\documentclass[twocolumn,10pt,letterpaper]{article}
\usepackage[top=2cm, bottom=2cm, left=1.75cm, right=1.75cm]{geometry}
\usepackage{booktabs} %
\usepackage{times}

\usepackage{amsthm}

\newcommand{\descr}[1]{\smallskip\noindent\textbf{#1}}
\usepackage[utf8]{inputenc}%
\newtheorem{definition}{Definition}
\newtheorem{theorem}{Theorem}
\usepackage[numbers,sort&compress]{natbib}
\usepackage[marginal,hang,stable]{footmisc}
\renewcommand{\footnoterule}{%
  \kern -3pt
  \hrule width 1in
  \kern 2pt
}

\makeatletter
\def\NAT@def@citea{\def\@citea{\NAT@separator}}
\makeatother
\newtheorem{lemma}{Lemma}

\usepackage{amsmath,amssymb,amsfonts}
\usepackage{tikz}
\usepackage{paralist}
\usepackage{graphicx}
\usepackage{graphics}
\usepackage{xspace}
\usepackage{setspace}
\usepackage{multirow}
\usepackage{color}
\usepackage[small,labelfont=bf]{caption}
\usepackage{indentfirst}
\usetikzlibrary{positioning}
\usepackage{stmaryrd}
\usepackage[ruled,vlined,linesnumbered]{algorithm2e} 
\SetAlCapNameFnt{\footnotesize}
\SetAlCapFnt{\footnotesize}
\usepackage[T1]{fontenc}

\usepackage{url}

\definecolor{darkgreen}{RGB}{47,109,79}
\definecolor{darkblue}{RGB}{57,79,99}

\newcommand{\transit}{TRANSIT\xspace}

\DeclareCaptionLabelFormat{andtable}{Figure#1~#2 \& Table~\thetable}

\usepackage{caption}
\usepackage{subcaption}

\makeatletter
\def\url@leostyle{%
 \@ifundefined{selectfont}{\def\UrlFont{}}%
 {\def\UrlFont{}}%
} \makeatother 
\usepackage[hyphenbreaks]{breakurl}

\usepackage[compact]{titlesec}
\titlespacing*{\section}{0pt}{*2}{3.5pt} 
\titlespacing{\subsection}{0pt}{*2}{4pt}
\titlespacing{\subsubsection}{0pt}{*1.5}{3pt}

\makeatletter
\def\NAT@def@citea{\def\@citea{\NAT@separator}}

\usepackage[bookmarks=false, colorlinks=true, plainpages = false, linkcolor = darkgreen, citecolor = darkgreen, urlcolor = darkblue, filecolor = blue]{hyperref}
\usepackage{cleveref}
\makeatletter
\def\@copyrightspace{\relax}
\makeatother
\begin{document} 

\pagestyle{plain}

\title{\bf Differentially Private Mixture of Generative Neural Networks\thanks{A shorter version of this paper appeared at the 17th IEEE International Conference on Data Mining (ICDM 2017). This is the full version, published in IEEE Transactions on Knowledge and Data Engineering (TKDE).}}

\author{Gergely Acs$^1$, Luca Melis$^2$, Claude Castelluccia$^3$, and Emiliano De Cristofaro$^2$\\
{\small $^1$CrySyS Lab, BME-HIT, $^2$University College London, $^3$INRIA}}

\date{}

\maketitle

\begin{abstract}
Generative models are used in a wide range of applications building on large amounts of contextually rich information. Due to possible privacy violations of the individuals whose data is used to train these models, however, publishing or sharing generative models is not always viable. In this paper, we present a novel technique for privately releasing generative models and entire high-dimensional datasets produced by these models. We model the generator distribution of the training data with a mixture of $k$ generative neural networks. These are trained together and collectively learn the generator distribution of a dataset. Data is divided into $k$ clusters, using a novel differentially private kernel $k$-means, then each cluster is given to separate generative neural networks, such as Restricted Boltzmann Machines or Variational Autoencoders, which are trained only on their own cluster using differentially private gradient descent. We evaluate our approach using the MNIST dataset, as well as call detail records and transit datasets, showing that it produces realistic synthetic samples, which can also be used to accurately compute arbitrary number of counting queries. 
\end{abstract}

\section{Introduction}

Generative models represent an emerging area of machine learning, as recent progress has made it possible to artificially generate plausible samples of various kinds of data, including images, videos, texts, and music.
These models are used, e.g. compression~\cite{TheisSCH17}, denoising~\cite{BengioYAV13}, inpainting~\cite{YehCLHD16}, super-resolution~\cite{LedigTHCATTWS16}, semi-supervised learning~\cite{SalimansGZCRCC16}, clustering~\cite{TheisOB15}, etc.
More specifically, generative models estimate  the underlying distribution of a dataset and randomly generate realistic samples
according to their estimated distribution.
The real distribution-generating data is described with significantly fewer parameters than the number of available samples from this distribution. This ``enforced compression'' incentives the model to describe general features of the training data. 

Ideally, such generalization should prevent the model from learning any individual-specific information. However, common algorithms often fail to provide such privacy guarantees and overfit on specific training samples by implicitly memorizing them. 
For example, in model inversion attacks~\cite{fredrikson2015model}, an adversary can use a trained model to make predictions of unintended (sensitive) attributes used as input to the model. 
Hence, even if only internal parameters are released, %
there might still be significant threats to the privacy of individuals whose data is used for training. 

In this paper, we present a novel approach supporting the privacy-preserving release of generative models.
While previous work explored the use of differential privacy in different areas of machine learning, %
including deep learning~\cite{abadi2016deep,PapernotAEGT16,shokri2015privacy}, privacy protection in generative models has not been explored  so far. 

\descr{Motivation.} Generative models play an important role whenever entities need to publish their datasets, e.g., aiming to monetize it or allow third parties with the appropriate expertise to analyze it. For instance, Call Detail Records (CDRs) collected by telecommunication companies are not only useful to capture interactions between customers, but also to understand their behavior, e.g., for infectious disease spreading or migration patterns.\footnote{See, e.g., \url{http://www.flowminder.org}.}
Rather than releasing only specific aggregate statistics, such as certain counting queries or histograms, one could share an ``anonymized'' dataset, which replaces the original data in possibly privacy-sensitive data analytics tasks. %
Alas, traditional anonymization models, such as $k$-anonymity, are not effective on high dimensional data, %
providing poor utility with insufficient privacy guarantees~\cite{aggarwal2005k}.

\descr{Intuition.} A more promising approach is to model the data generating distribution by training a generative model on the original data, and only publish the model along with its (differentially private) parameters. Provided with this privacy-preserving model, anybody can generate a {\em synthetic} dataset resembling the original (training) data as much as possible without violating the strong protection of differential privacy.
The intuition is that generative models have the potential to automatically learn the general features of a dataset including complex regularities such as the subtle and valuable correlation among different attributes.

\descr{Overview of the solution.} Following this intuition, we propose a generative model that is a mixture of $k$ 
generative artificial neural networks (ANNs). These ANNs are trained together and collectively learn the generator distribution of a dataset. 
The data is first divided into $k$ clusters using a differentially private clustering approach, then each cluster is given to a separate generative neural network, such as Restricted Boltzmann Machines (RBM)~\cite{Goodfellow-et-al-2016-Book} or Variational Autoencoders (VAE)~\cite{kingma2013auto}, which are trained only on their own cluster using differentially private gradient descent. 

Training distinct generative models on different partitions of the dataset has several benefits.  
First, multiple models can generate more accurate synthetic samples than a single model trained on the whole dataset, as each
ANN is trained only on similar data samples. This prevents the mixture model to generate unrealistic synthetic samples which may arise from the implausible combination of multiple very different clusters. This scenario is much more likely when the training is perturbed to guarantee differential privacy. 
Second, each ANN models a different component of the generator distribution, and hence learn any specifics of a cluster faster than a single model. In other words, a single model would need more training epochs than a mixture of generative models to achieve a comparably rich representation of the clusters.     
As each iteration of the learning algorithm requires some perturbation to guarantee privacy, a mixture  model needs less noise which eventually yields more accurate model parameters.

\descr{Privacy Guarantees.} Overall, our work builds on the Differential Privacy (DP) framework, specifically, using the Gaussian mechanism~\cite{Dwork2014book}. 
For clustering, we use a novel differentially private kernel $k$-means algorithm; kernel $k$-means~\cite{ScholkopfSM98} is a non-linear extension of the classical $k$-means algorithm and has been shown to be equivalent with most other kernel based clustering algorithms~\cite{DhillonGK04}. We first transform the data into a low-dimensional space using random Fourier features~\cite{RahimiR07}, and then apply a differentially private version of Lloyd's algorithm~\cite{BlumDMN05} to find the clusters in the data. Random Fourier features does not only make kernel $k$-means scalable for large datasets~\cite{ChittaJJ12}, but, unlike standard $k$-means~\cite{BlumDMN05}, require to add limited amount of noise to guarantee privacy.
Finally, when clusters are created, a generative model is trained on each cluster using differentially private stochastic gradient descent (SGD), which is a standard learning technique of many generative ANNs.  Previous work adds constant amount of noise to the gradient update in each SGD iteration to guarantee differential privacy., whereas, we add noise to each gradient update which is tailored to the data.  

We prove that our scheme guarantees differential privacy by using the \emph{moment accountant} method, proposed in~\cite{abadi2016deep}, which allows to quantify the privacy guarantee of the composition of differentially private mechanisms (e.g., noisy $k$-means iterations followed by noisy SGD iterations) much more accurately than previous work~\cite{dwork2010boosting}.

\descr{Contributions.} In summary, we make several contributions:
\begin{enumerate}
\item We propose a novel approach, relying on generative neural networks, to model the data generating distribution of various kinds of data. It provides differential privacy to each individual in the training data, thus, it can be used to effectively ``anonymize'' and share large high-dimensional datasets with any potentially adversarial third party. 
\item We design a novel differentially private clustering algorithm, combining kernel $k$-means with random Fourier features, which efficiently clusters high-dimensional large datasets with strong privacy guarantees.
\item We present a Differentially Private Generative Model (DPGM), where data is first clustered, using the differentially private kernel $k$-means, and then each cluster is given to separate generative neural networks, such as Restricted Boltzmann Machines or Variational Autoencoders, which are trained only on their own cluster using differentially private gradient descent.
\item We improve the differentially private gradient descent algorithm by Abadi et al.~\cite{abadi2016deep}, using a novel adaptive perturbation technique. We adaptively re-compute the magnitude of the noise used to perturb the gradient updates in each SGD iteration, which can lead to significant accuracy improvement of the trained model. 
\item We evaluate our approach on the MNIST dataset~\cite{lecun1998gradient}, a large Call Detail Records (CDR), and a transit dataset (\transit); we show that our techniques provide realistic synthetic samples which can also be used to accurately compute arbitrary number of counting queries.
\end{enumerate}

\section{Related work}
In this section, we review prior work on privacy-preserving mechanisms applied to data mining, machine learning, and deep learning.

\descr{Private Data Release.}
The $k$-anonymity~\cite{Sweeney02} paradigm aims to protect data by generalizing and suppressing certain identifying attributes, however, it does not work well on high-dimensional datasets~\cite{aggarwal2005k,brickell2008cost}.
Therefore, rather than pursuing input sanitization, prior work has proposed techniques to produce plausible synthetic records with strong privacy guarantees, e.g., focusing on differentially private release of data~\cite{abowd2008protective,charest2011can,chen2011publishing,jagannathan2008privacy,machanavajjhala2008privacy,mcclure2012differential,wasserman2010statistical}. Alas, these can often support only the release of succinct data representations, such as histograms or contingency tables.

Other mechanisms add noise directly to a generative model~\cite{bowen2016differentially,li2014differentially,liu2016model,zhang2014privbayes}. In this paper, we follow this approach, while, in a first-of-its-kind attempt, focusing on building private generative machine learning models based on neural networks.
Other approaches~\cite{bindschaedler2017plausible,reiter2009estimating,reiter2014bayesian} generate data records first, and then attempt to test their privacy guarantees, i.e., decoupling the generative model from the privacy mechanism. 
By contrast, we attempt to achieve privacy during the training of the model, thus avoiding eventual high sample rejection rates due to privacy tests.

\descr{Privacy in Deep Learning.}
Our work builds on the Differential Privacy (DP) framework, specifically, using the Gaussian mechanism~\cite{Dwork2014book}. Due to its generality, DP has served as a building block in several recent efforts at the intersection of privacy and machine learning~\cite{abadi2016deep,shokri2015privacy}.
In general, the majority of privacy-preserving learning schemes focus on convex optimization problems~\cite{bassily2014private,chaudhuri2011differentially,wu2016differentially}. 
Also, training neural networks typically requires to optimize non-convex objective functions -- as with Restricted Boltzmann Machine (RBM)~\cite{carlson2015stochastic} and Variational Autoencoder (VAE)~\cite{kingma2013auto} -- which is usually done through the application of Stochastic Gradient Descent (SGD) with poor theoretical guarantees.
Wu et al.~\cite{wu2016differentially} introduce a privacy-preserving technique which runs SGD for convex cases for a constant number of iterations and only adds noise to the final output.
By contrast, we introduce a novel differentially private SGD algorithm for optimizing general non-convex loss functions.

Shokri et al.~\cite{shokri2015privacy} support distributed training of deep learning networks in a privacy-preserving way. Specifically, their system relies on the input of independent entities which aim to collaboratively build a machine learning model without sharing their training data. To this end, they selectively share subsets of noisy model parameters during training. 
However, their approach incurs high levels of privacy loss per entity, i.e., the $\varepsilon$ parameter is in the order of thousands, using the strong composition theorem~\cite{dwork2010boosting}.
Abadi et al.~\cite{abadi2016deep} introduce an algorithm for non-convex deep learning models with strong differential privacy guarantees.
They propose a privacy accounting method, called the moments accountant, which guarantees a tighter bound of the privacy loss for the composition of multiple gaussian mechanisms when compared to the strong composition theorem~\cite{dwork2010boosting}. 
Our method also relies on the moments accountant to measure privacy loss, but we train generative models (i.e., unsupervised learning) and with an improved gradient descent, where the noise is carefully adjusted and injected in each iteration. %

Also, Beaulieu et al.~\cite{beaulieu2017privacy} apply the noisy gradient descent from~\cite{abadi2016deep} to train the discriminator of a Generative Adversarial Network~\cite{goodfellow2014generative} under differential privacy. The resulting model is then used to generate synthetic subjects based on the population of clinical trial data.
In this paper, we rather use Variational Autoencoder which is trained with an improved version of the noisy gradient descent. Also,
we apply a private clustering technique on the training data and create multiple generative models that produce higher-quality synthetic samples compared to a single model.

\descr{Differentially Private k-means} has also been studied in prior work~\cite{SuCLBJ16}, however, aiming to find linearly separable clusters and add noise which is proportional to the data dimension $m$ or the $L_1$-norm of data records. By contrast, our private kernel $k$-means approach can find even linearly non-separable clusters, and the added noise is independent of $d$ as well as the norm of data points. Also, we offer a tighter privacy analysis using the moments accountant method from~\cite{abadi2016deep}. 
Kernel $k$-means clustering with random Fourier features (RFF) has already been considered in~\cite{ChittaJJ12}, albeit without any privacy guarantee. 
We somewhat combine~\cite{ChittaJJ12} and~\cite{BlumDMN05}, applying DP $k$-means on Fourier features to ultimately achieve better accuracy than~\cite{BlumDMN05}.

\descr{Clustering and Generative Neural Networks.} Prior work has also attempted to combine clustering with deep learning, though with no privacy guarantees.
Some proposals~\cite{zheng2016variational,huang2014deep,yang2016towards} jointly train an autoencoder neural network with a clustering algorithm, and use the internal representation provided by the autoencoder, i.e., the encoder output, as features for clustering.
A different training method is used in~\cite{li2017discriminatively,dizaji2017deep,xie2016unsupervised}, where autoencoders are initially pre-trained, and then fine-tuned using the cluster assignment loss.
Finally, other techniques~\cite{hsu2018cnn,yang2016joint} combine clustering with standard convolutional neural networks (CNNs) for representation learning of images.

\section{Preliminaries}

In this section, we review concepts used throughout the rest of the paper.
We use the following notation: $\mathbb{I}$ denotes a universe of items (e.g., set of visited locations, pixels in an image, etc.), where $|\mathbb{I}| = m$.
A dataset $D\subseteq 2^\mathbb{I}$ is the ensemble of all items of some set of individuals. %
A record, which is a non-empty subset of $\mathbb{I}$, refers to all items of an individual from $D$ and is represented by a binary vector $\mathbf{x}$ of size $m$.

\subsection{Restricted Boltzmann Machines (RBM)}
\label{sec:rbm}

A Restricted Boltzmann Machine (RBM) is a bipartite undirected graphical model %
composed of $m$ visible and $n$ invisible (or latent) binary random variables denoted by, respectively, $\mathbf{v} = (v_1, v_2, \ldots, v_m)$ and $\mathbf{h} = (h_1, h_2, \ldots, h_n)$.
In our case, visible variables represent the attributes of $D$ and their values are composed of records from $D$.
Hidden variables capture the dependencies between different visible variables (i.e., dependencies between the items in $\mathbb{I}$). As the above model is a Markov random field with strictly positive joint probability distribution $p$ over the model variables, $p$ can be represented as a Boltzmann distribution defined as: %
\begin{align}
\label{eq:gibbs}
p(\mathbf{v}, \mathbf{h}) = \frac{1}{Z}e^{-E(\mathbf{v}, \mathbf{h})} %
\end{align}
\noindent where $Z= \sum_{\mathbf{v}, \mathbf{h}} e^{-E(\mathbf{v}, \mathbf{h})}$ is the partition function, 
$E(\mathbf{v}, \mathbf{h})$ the energy function, i.e., $E(\mathbf{v}, \mathbf{h})= - \sum_{i=1}^n \sum_{j=1}^{m}v_{ij}h_iv_j -\sum_{j=1}^m b_j v_j - \sum_{i=1}^n c_i h_i$,
with $w_{ij}$ being real valued weights describing the inter-dependency between $v_j$ and $h_i$, and $b_j, c_i$ real valued bias terms associated with the $j$th visible and $i$th hidden units, respectively. Using matrix notation, $E(\mathbf{v}, \mathbf{h}) = -\mathbf{v}^\top\mathbf{W}\mathbf{h} - \mathbf{b}^\top \mathbf{v} - \mathbf{c}^\top\mathbf{h}$, where $\mathbf{W} = \llbracket w \rrbracket _{i,j}$, $\mathbf{c} = [c]_i$, and $\mathbf{b} = [b]_j$.
The goal is to approximate the true data generating distribution with the Boltzmann distribution $p$, given in Eq.~\eqref{eq:gibbs}.
To this end, we train the RBM model on dataset $D$ to compute parameters $\mathbf{W}, \mathbf{c}, \mathbf{b}$.

There are a few algorithms to train RBMs, that approximate or relate to gradient descent on the log-likelihood of the data. 
If $\theta = (\mathbf{W}, \mathbf{b}, \mathbf{c})$, then we want to maximize the likelihood function $\mathcal{L}(\theta| D) = \prod_{\mathbf{x} 
\in D} p(\mathbf{x} | \theta) $ given dataset $D$, where $\mathbf{x} \in \{0,1\}^m$ is a record from $D$ and $p$ is the Boltzmann distribution defined in Eq.~\eqref{eq:gibbs}. 
A numerical approximation, gradient descent, is used where the model parameters $\theta$ are iteratively updated using $D$ and the gradient of the log-likelihood function as: 
$\theta_{t+1} = \theta_{t} + \eta \frac{\partial \log\mathcal{L}(\theta_t|D) }{\partial \theta_t}$,
with $\eta \in \mathbb{R}^+$ being the learning rate. The model parameters are updated until the log-likelihood converges. In this paper, we employ Persistent Contrastive Divergence~\cite{tieleman2008training}.

\subsection{Variational Autoencoder (VAE)}
\label{sec:VAE}

A variational autoencoder~\cite{kingma2013auto} consists of two neural networks (an \emph{encoder} and a \emph{decoder}), and a loss function.
The encoder compresses data into a latent space ($z$) while the decoder reconstructs the data given the hidden representation.
Let $\mathbf{x}$ be a random vector of $m$ observed variables, which are either discrete or continuous. 
Let $\mathbf{z}$ be a random vector of $n$ latent continuous variables.
The probability distribution between $\mathbf{x}$ and $\mathbf{z}$ assumes the form $p_\theta(\mathbf{x}, \mathbf{z}) = p_\theta(\mathbf{z}) p_\theta(\mathbf{x} \mid \mathbf{z})$, where $\theta$ indicates that $p$ is parametrized by $\theta$.
Also, let $q_\phi(\mathbf{z} \mid \mathbf{x})$ be a recognition model whose goal is to approximate the true and intractable posterior distribution $p_\theta(\mathbf{z} \mid \mathbf{x})$.
We can then define a lower-bound on the log-likelihood of $\mathbf{x}$ as follows:
$ \mathcal{L}(\mathbf{x}) = - D_{KL}(q_\phi(\mathbf{z} \mid \mathbf{x}) \mid\mid p_\theta(\mathbf{z}))
	+ \mathrm{E}_{q_\phi(\mathbf{z} \mid \mathbf{x})} [\log p_\theta(\mathbf{x} \mid \mathbf{z})]$.
The first term pushes $q_\phi(\mathbf{z} \mid \mathbf{x})$ to be similar to $p_\theta(\mathbf{z})$ ensuring that, while training, VAE learns a decoder that, at generation time, will be able to invert samples from the prior distribution such they look just like the training data.
The second term can be seen as a form of reconstruction cost, and needs to be approximated by sampling from $q_\phi(\mathbf{z} \mid \mathbf{x})$.

In VAEs, we propagate the gradient signal through the sampling process and through $q_\phi(\mathbf{z} \mid \mathbf{x})$ using the \emph{reparametrization trick}.
This is done by making $\mathbf{z}$ be a deterministic function of $\phi$ and some noise $\mathbf{\epsilon}$, i.e., 
$\mathbf{z} = f(\phi, \mathbf{\epsilon})$.
For instance, sampling from a normal distribution can be done like $\mathbf{z}=\mu + \sigma \mathbf{\epsilon}$, where $\mathbf{\epsilon} \sim \mathcal{N}(0, \mathbf{I})$.
The reparametrization trick can be viewed as an efficient way of \emph{adapting} $q_\phi(\mathbf{z} \mid \mathbf{x})$ to help improve the reconstruction.
We train the Variational AutoEncoder using stochastic gradient descent to optimize the loss with respect to the parameters of the encoder and decoder $\theta$ and $\phi$.

\subsection{Kernel k-means with Random Features}
\label{sec:kkmeans}

 Given a set of samples $D = \{ \mathbf{x}_1,\mathbf{x}_2, \ldots, \mathbf{x}_N$\}, $k$-means linearly separates $D$ into $k $ clusters $C_1, C_2, \ldots, C_k$ $(k \leq N)$ so that it aims to minimize the error $\sum_{i=1}^k \sum_{\mathbf{x} \in C_i} || \mathbf{x} - \mathbf{c}_i ||_2^2$, where $\mathbf{c}_i = \sum_{\mathbf{x} \in C_i} \mathbf{x}/ |C_i|$  is the centroid of cluster $C_i$. Although this problem is NP-hard,  there are efficient heuristic algorithms (such as Lloyd's algorithm) which iteratively refines clustering and converge quickly to a local optimum. 
However, $k$-means can provide very inaccurate clustering of linearly \emph{non}-separable data, which are very common in practice.  To overcome this shortcoming, kernel $k$-means~\cite{ScholkopfSM98} first maps  samples from input space to a higher dimensional feature space through a non-linear transformation $\Phi$, then applies  standard $k$-means on $\{ \Phi(\mathbf{x}_1), \Phi(\mathbf{x}_2), \ldots, \Phi(\mathbf{x}_N)\}$. Hence, kernel $k$-means provides linear separators of clusters in feature space which correspond to non-linear separators in input space. 
Kernel $k$-means iteratively computes $|| \Phi(\mathbf{x}) - \mathbf{c}_i'||_2^2$ for each sample $\mathbf{x}$ to decide which cluster a sample belongs to, where $\mathbf{c}'_i = \sum_{\mathbf{x} \in C_i} \Phi(\mathbf{x})/ |C_i|$. To do so, the inner product $\langle\Phi(\mathbf{x}), \Phi(\mathbf{y}) \rangle$ must be known for all $\mathbf{x}, \mathbf{y} \in D$. Since $\Phi(\cdot)$ is hard to explicitly compute due to its large, often infinite dimension, the kernel trick is applied;  $\langle\Phi(\mathbf{x}), \Phi(\mathbf{y}) \rangle = \kappa(\mathbf{x}, \mathbf{y})$, where $\kappa$ is an easily computable  kernel function. Still, this approach requires evaluating $\kappa$ for all pairs of samples and store the results, which is not scalable for large datasets.

To make kernel $k$-means scalable, the kernel function can be approximated with low-dimensional explicit feature maps. In particular, 
the samples are first mapped to a low-dimensional Euclidean inner product space using an explicit random feature map $z : 
\mathbb{R}^m \rightarrow \mathbb{R}^d$ so that $\langle \Phi(\mathbf{x}), \Phi(\mathbf{y}) \rangle \approx \langle z(\mathbf{x}), z(\mathbf{y}) \rangle $. Then, standard $k$-means is applied on the low-dimensional mapped samples $ \{ z(\mathbf{x}_1), z(\mathbf{x}_2), \ldots, z(\mathbf{x}_N)\}$ in $\mathbb{R}^d$  to approximate the result of the kernel $k$-means with implicit feature map $\Phi$ and kernel $\kappa$. The approximation error decreases exponentially fast as $d$ increases, and quite accurate approximations can be obtained  even for relatively small $d$. In particular, the approximation error is less than $\xi$ with only $d =O(m\xi^{-2}\log \xi^{-2})$ dimensions~\cite{RahimiR07}. 
Explicit nonlinear feature maps have already been proposed for  shift-invariant kernels  (e.g., generalized RBF kernels)~\cite{VempatiVZJ10} as well as polynomial kernels~\cite{PenningtonYK15} among others.

\subsection{Differential Privacy (DP)}
\label{sec:DP}
Differential Privacy allows a party to privately release a dataset: using perturbation mechanisms, a function of an input dataset is modified, so that any information which can discriminate a record from the rest of the dataset is bounded~\cite{Dwork2014book}.
 \begin{definition}[Privacy loss]
 Let $\mathcal{A}$ be a privacy mechanism which assigns a value $\mathit{Range}(\mathcal{A})$ to a dataset $D$. The privacy loss of $\mathcal{A}$ with datasets $D$ and $D'$ at output $O \in \mathit{Range}(\mathcal{A})$ is a random variable $\mathcal{P}(\mathcal{A},D,D',O) = \log\frac{\Pr[\mathcal{A}(D) = O]}{\Pr[\mathcal{A}(D') = O]}$ 
 where the probability is taken on the randomness of $\mathcal{A}$.%
 \label{def:ploss}
 \end{definition}
\smallskip

\begin{definition}[$(\epsilon,\delta)$-Differential Privacy~\cite{Dwork2014book}] 
A privacy mechanism $\mathcal{A}$ guarantees $(\varepsilon, \delta)$-differential privacy if for any database $D$ and $D'$, differing on at most one record, and for any possible output $S \subseteq \mathit{Range}(\mathcal{A})$, 
$Pr[\mathcal{A}(D)\in S]\leq e^{\varepsilon}\times Pr[\mathcal{A}(D') \in S] + \delta$
or, equivalently, $\Pr_{O \sim \mathcal{A}(D)}[\mathcal{P}(\mathcal{A},D,D',O) > \varepsilon] \leq \delta$. 
\label{def:DP}
\end{definition}

This definition guarantees that every output of algorithm $\mathcal{A}$ is almost equally likely (up to $\varepsilon$) on datasets differing in a single record except with probability at most $\delta$, preferably smaller than $1/|D|$. Intuitively, this guarantees that an adversary, provided with the output of $\mathcal{A}$, can draw almost the same conclusions about any individual no matter if this individual is included in the input of $\mathcal{A}$ or not~\cite{Dwork2014book}.

Differential privacy maintains composition, i.e., if each of $\mathcal{A}_1, \ldots, \mathcal{A}_k$ is $(\varepsilon, \delta)$-DP, then their $k$-fold adaptive composition\footnote{The output of $\mathcal{A}_{i-1}$ is used as input to $\mathcal{A}_i$, i.e., their executions are not necessarily independent except their coin tosses.} is $(k\varepsilon, k\delta)$-DP. However, a tighter upper bound can be derived on the privacy loss of the composite using a generic Chernoff bound. In particular, it follows from Markov's inequality that $\Pr[\mathcal{P}(\mathcal{A},D,D',O) \geq \varepsilon] \leq \mathbb{E}[\exp(\lambda \mathcal{P}(\mathcal{A},D,D',O))]/\exp(\lambda\varepsilon)$ for any output $O \in \mathit{Range}(\mathcal{A})$ and $\lambda > 0$. This implies that $\mathcal{A}$ is $(\varepsilon, \delta)$-DP %
with $\delta = \min_{\lambda} \exp(\alpha_{\mathcal{A}}(\lambda) - \lambda \varepsilon)$, where $\alpha_{\mathcal{A}}(\lambda) = \max_{D,D'} \log\mathbb{E}_{O\sim \mathcal{A}(D)}[\exp(\lambda \mathcal{P}(\mathcal{A},D,D',O))]$ is the log of the moment generating function of the privacy loss. %
\begin{theorem}[Moments accountant~\cite{abadi2016deep}]
	Let $\alpha_{\mathcal{A}_i}(\lambda) $ be $ \max_{D,D'} \log\mathbb{E}_{O \sim \mathcal{A}(D)}[\exp(\lambda \mathcal{P}(\mathcal{A},D,D',O))]$ and $\mathcal{A}_{1:k}$ the $k$-fold adaptive composition of $\mathcal{A}_1, \mathcal{A}_2, \ldots, \mathcal{A}_{k}$. It holds: 
	\begin{compactenum}
		\item $\alpha_{\mathcal{A}_{1:k}}(\lambda) \leq  \sum_{i=1}^k \alpha_{\mathcal{A}_i}(\lambda)$
	\item $\mathcal{A}_{1:k}$ is $(\varepsilon, \min_{\lambda}\exp(\sum_{i=1}^k\alpha_{\mathcal{A}_i}(\lambda) - \lambda \varepsilon))$-differentially private
	\end{compactenum}
where $\mathcal{A}_1, \mathcal{A}_2, \ldots, \mathcal{A}_{k}$ use independent coin tosses. %
	\label{thm:comp}
\end{theorem}

\noindent There are a few ways to achieve DP, including the Gaussian mechanism~\cite{Dwork2014book}. A fundamental concept of all of them is the \emph{global sensitivity} of a function~\cite{Dwork2014book}.

\begin{definition}[Global $L_p$-sensitivity] 
For any function $f:\mathcal{D} \rightarrow \mathbb{R}^ d$, the $L_p$-sensitivity of $f$ is
$\Delta_p f = \max_{D, D'} || f(D)-f(D') ||_p$, 
for all $D, D'$ differing in at most one record, where $||\cdot||_p$ denotes the $L_p$-norm.\vspace*{-0.15cm}
\label{def:global_sens}
\end{definition}

 \begin{figure*}[!t]
	\centering
		\includegraphics[width=0.8\textwidth]{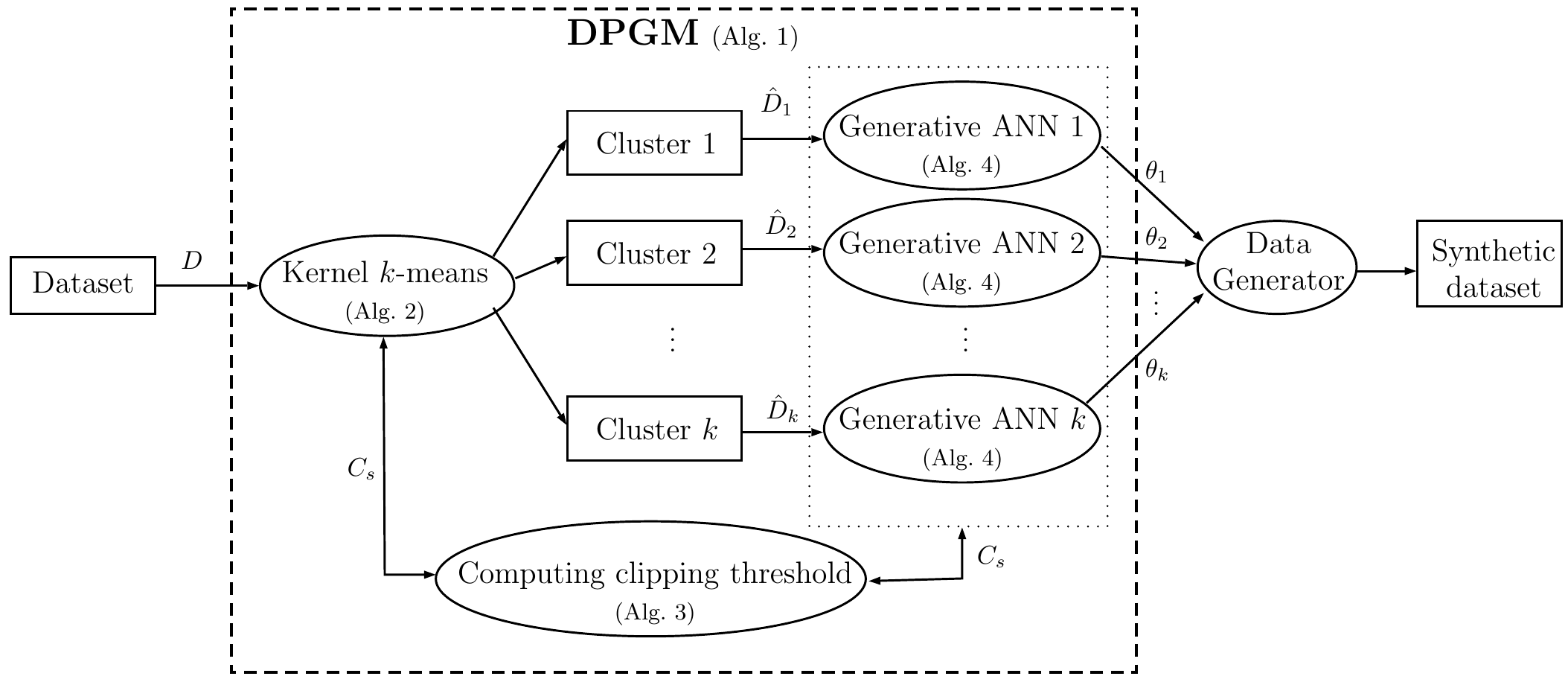}
	\vspace{-0.1cm}
		\caption{Overview of our differentially private generative model (DPGM).}
		\label{fig:scheme}		
\end{figure*}

\descr{Gaussian Mechanism.} The Gaussian Mechanism (GM)~\cite{Dwork2014book} 
consists of adding Gaussian noise to the true output of a function. In particular, for any function $f:\mathcal{D} \rightarrow \mathbb{R}^ d$, GM is defined as $\mathcal{G}(D) = f(D) + [ \mathcal{N}_1(0,\Delta_2 f \cdot \sigma), \ldots, \mathcal{N}_d(0,\Delta_2 f \cdot  \sigma)]$, where $\mathcal{N}_i(0,\Delta_2 f \cdot \sigma)$ are i.i.d.~normal random variables with zero mean and variance $(\Delta_2 f \cdot \sigma)^2$. %

\begin{lemma} 
	$\alpha_{\mathcal{G}}(\lambda) = (\lambda^2 + \lambda)/4\sigma^2$
	\label{lem:gauss_alpha}\vspace{-0.1cm}
\end{lemma}

\begin{proof}
Let $f: \mathcal{D} \rightarrow \mathbb{R}$ be a scalar function, $f(D) =  f(D') + \Delta_1 f$, where $\Delta_1 f = \Delta_2f$, and $O = f(D)+x$, where $x \sim \mathcal{N}(0, \sigma)$. Let $\hat{\sigma} = \Delta_1 f \cdot \sigma$ Then, it holds: 
{\footnotesize
\begin{multline}
\notag \mathcal{P}(\mathcal{A},D,D',O) = \ln\left( \frac{\Pr[\mathcal{G}(D) = O]}{\Pr[\mathcal{G}(D')  = O]}\right) =\\
= \ln\left( \frac{\Pr[f(D) + \mathcal{N}(0,\hat{\sigma}) = O]}{\Pr[f(D') + \mathcal{N}(0,\hat{\sigma}) = O]}\right) = \ln\left( \frac{\exp(-x^2/2\hat{\sigma}^2)}{\exp(-(x+\Delta_1 f)^2/2\hat{\sigma}^2)}\right) = \\
=  \ln\left( \frac{\exp(-x^2/2\hat{\sigma}^2)}{\exp(-(x+\Delta_1 f)^2/2\hat{\sigma}^2)}\right) =  \left( \frac{\Delta_1 f}{\hat{\sigma}}\cdot \frac{x}{\hat{\sigma}}\right) + \frac{1}{2} \left( \frac{\Delta_1 f}{\hat{\sigma}} \right)^2 
\end{multline}
}
\noindent Since $x$ is drawn from $\mathcal{N}(0, \hat{\sigma})$, $\mathcal{P}(\mathcal{A},D,D',O)$ follows a normal distribution with mean $(\Delta_1 f)^2/2\hat{\sigma}^2$ and standard deviation $\Delta_1 f / \hat{\sigma}$, whose moment generating function  is $\exp\left((\lambda^2 + \lambda)(\Delta_1 f)^2/4\hat{\sigma}^2 \right)$. The claim follows from the definition of $\alpha$ and $\hat{\sigma}$. For the high-dimensional case when $f: \mathcal{D} \rightarrow \mathbb{R}^d$ $(d>1)$, the proof is similar to that of Theorem A.1 in \cite{Dwork2014book}.
\end{proof}

\smallskip
Given $\alpha_{\mathcal{G}}(\lambda)$, the exact privacy cost $\varepsilon$ (or $\delta$) of the $k$-fold adaptive composition of $\mathcal{G}$ is computed based on Theorem~\ref{thm:comp}.

\section{DPGM: Differentially Private Generative Model} %
In this section, we present our Differentially Private Generative Model (DPGM), which is detailed in Alg.~\ref{alg:ours} and illustrated in \figurename~\ref{fig:scheme}. Table~\ref{tab:notation} summarizes notation and symbols used throughout the paper.
The dataset $D$ is first partitioned into $k$ clusters, denoted by $\hat{D}_1, \hat{D}_2, \ldots, \hat{D}_k$, which are in turn used to train $k$ distinct generative models, where the parameters of the resulting models are denoted, respectively, by $\theta_1, \theta_2, \ldots, \theta_k$. 
Data samples are similar within a cluster, thus, generative models simultaneously trained on each partition converge faster than a single model trained on the whole dataset $D$.
As $\theta_1, \theta_2, \ldots, \theta_k$ are learnt using perturbed gradient descent, they can be released and used to generate synthetic data using the $k$ generative models. 

Our learning approach involves two main steps: 
\begin{enumerate}
\item Records in $D$ are clustered in a random feature space using differentially private kernel $k$-means (see Section~\ref{sec:DPKMEANS}) into clusters $\hat{D}_1, \hat{D}_2, \ldots, \hat{D}_k$; and
\item A generative model (e.g., RBM~\cite{Goodfellow-et-al-2016-Book} or VAE~\cite{kingma2013auto}) with parameter $\theta_i$ is trained on cluster $\hat{D}_i$ (see Section~\ref{sec:DPSGD}) using differentially private gradient descent, where the training data are composed of the records of $\hat{D}_i$. 
\end{enumerate}
In each SGD iteration (Line 4-6 in Alg.~\ref{alg:ours}), a model $\theta_s$ is chosen uniformly at random along with corresponding training data $\hat{D}_s$, and a single SGD iteration is performed to update $\theta_s$ using a random sample $S$ of $\hat{D}_s$ with size $L$ (Line 7 in Alg.~\ref{alg:ours}). %
The output of our algorithm are the parameters of the trained generative models, i.e., $\theta_1, \theta_2, \ldots, \theta_k$. 
Finally, these privately trained $k$ models can be used to generate synthetic records which resemble the original ones, i.e., preserve their general characteristics that are not specific to any single individual (as per $\varepsilon$ and $\delta$ discussed in Section \ref{sec:analysis}). 

\begin{algorithm}[t]
\small
	\caption{DPGM: Differentially Private Generative Model\label{alg:ours}}
	\DontPrintSemicolon
	\KwIn{\emph{Dataset:} $D = \{\mathbf{x}_1,\dots,\mathbf{x}_N\}$, \emph{\# of custers:} $k$, \emph{$k$-means iterations:} $T_{\mathcal{K}}$, \emph{SGD iterations:} $T_{\mathcal{S}}$, \emph{Noise scales:} $\sigma_{\mathcal{C}}, \sigma_{\mathcal{K}}, \sigma_{\mathcal{G}}$}  
	{\bf Cluster data records in $D$:}
    $\{\hat{D}_1, \hat{D}_2, \ldots, \hat{D}_k\}= \mathrm{DPkmeans}(k, T_{\mathcal{K}}, D, \sigma_{\mathcal{C}}, \sigma_{\mathcal{K}})$\; 
    	\textbf{Initialize} $\theta_1, \theta_2, \ldots, \theta_k$ randomly\;
   	\For {$t \in [T_{\mathcal{S}}]$}
   {
   	Select $(\hat{D}_s, \theta_s) \in \{(\hat{D}_1, \theta_1), \ldots, (\hat{D}_k,\theta_k)\}$ with probability $|\hat{D}_s| / |D|$\;
   	{\bf Update parameters of model  $\theta_s$:}\;
   $\theta_s = \textrm{DP-SGD}(\hat{D}_s, \theta_s, \sigma_{\mathcal{C}}, \sigma_{\mathcal{G}})$ \texttt{//see Alg.~\ref{alg:dpsgd}}
   
     }
	\KwOut{$\theta_1, \theta_2, \ldots, \theta_k$} 
\end{algorithm}

\begin{table}[!t]
\centering
\footnotesize
\begin{tabular}{|l|l|}
\hline
\textbf{Symbol} & \textbf{Description} \\
\hline
\hline
$\mathbf{x}$ & binary vector \\
$D$ & dataset \\
$k$ & number of $k$-means clusters \\
$T_{\mathcal{K}}$ & $k$-means iterations \\
$\hat{D}_1, \ldots, \hat{D}_k$ & data clusters \\
$\theta_1, \ldots, \theta_k$ & generative models \\ 
$\mathbf{\hat{c}}_1, \ldots, \mathbf{\hat{c}}_k$ & noisy cluster centers \\
$\sigma_{\mathcal{C}}, \sigma_{\mathcal{K}}, \sigma_{\mathcal{G}}$ & noise scales \\
$C_{\max}$ & max. norm bound \\
$w$ & max. number of discretized norm bounds \\
$\kappa$ & kernel function \\
z & randomized Fourier feature map\\
$d$ & number of features \\
$C_s$ & clipping threshold \\
$\mathcal{L}$ & loss function \\
$\eta$ & learning rate \\
$L$ & batch size \\
\hline
\end{tabular}
\vspace*{-0.1cm}
\caption{Notation and symbols used throughout the paper.\label{tab:notation}}
\vspace*{-0.1cm}
\end{table}

\subsection{Private kernel k-means}
\label{sec:DPKMEANS}
We now discuss our private kernel $k$-means algorithm, presented in Alg.~\ref{alg:dpkkmeans}. It first transforms the data $D$ into a low-dimensional representation $D' = \{z(\mathbf{x}_1),\dots,z(\mathbf{x}_N)\}$ using  randomized Fourier feature map $z: \mathbb{R}^m \rightarrow \mathbb{R}^d$~\cite{RahimiR07}, and then applies standard differentially private $k$-means~\cite{BlumDMN05} on these  low-dimensional features. 

Specifically, $z : \mathbb{R}^m \rightarrow \mathbb{R}^d$ is defined as: 
\begin{align}
\label{eq:z}
z(\mathbf{x}) =  \sqrt{\frac{2}{d}} \left[\cos(\langle \mathbf{w}_1, \mathbf{x} \rangle  + b_1), \ldots,  \cos(\langle \mathbf{w}_d, \mathbf{x} \rangle  + b_d )\right] 
\end{align}
\noindent where  each $\mathbf{w}_i \in \mathbb{R}^m$ is drawn independently from $p(\mathbf{w})= \frac{1}{2\pi} \int_{\mathbb{R}^m} \exp(-j \langle \mathbf{w}, \mathbf{x} \rangle) \kappa(\mathbf{w})d \mathbf{x}$, i.e., $p(\mathbf{w})$ is the Fourier transform of kernel function $\kappa$, and $b_i \in \mathbb{R}$ is chosen from $[0,2\pi)$ uniformly at random.
In particular, Bochner’s theorem implies that $p(\mathbf{w})$ is a valid probability density function, if $\kappa$  is continuous, positive-definite, and shift-invariant kernel. Hence, $\kappa(\mathbf{x}, \mathbf{y}) = \kappa(\mathbf{x} - \mathbf{y}) =  \int_{\mathbb{R}^m} \exp(j \langle \mathbf{w}, \mathbf{x}- \mathbf{y}\rangle) p(\mathbf{w})d \mathbf{w} = \mathbb{E}_{\mathbf{w}, b}[\langle \sqrt{2}\cos(\langle\mathbf{w},\mathbf{x}\rangle + b), \sqrt{2}\cos(\langle\mathbf{w},\mathbf{y}\rangle + b) \rangle ]$, where the expectation is approximated with the empirical mean over $d$ randomly chosen values of $\mathbf{w}$ and $b$~\cite{RahimiR07}.

\begin{algorithm}[t]
\small
		\caption{DPkmeans: Private kernel $k$-means with Random Fourier Features \label{alg:dpkkmeans}}
	\DontPrintSemicolon
	\KwIn{\emph{Data:} $D = \{\mathbf{x}_1,\dots,\mathbf{x}_N\}$, \emph{Cluster number:} $k$, \emph{Iterations:} $T$, \emph{Feature number:} $d$, \emph{Kernel function:} $\kappa$, \emph{Noise scales:} $\sigma_{\mathcal{C}}, \sigma_{\mathcal{K}}$}
	{\bf Compute Features:}
	$\mathbf{w}_i \sim_{\text{iid}} p(\mathbf{w})$ for all $1 \leq i \leq d$, where $p(\mathbf{w})= \frac{1}{2\pi} \int_{\mathbb{R}^m} \exp(-j \langle \mathbf{w}, \mathbf{x} \rangle) \kappa(\mathbf{w})d \mathbf{x}$\;
    $b_i \sim_{\text{iid}} \mathcal{U}[0,2\pi]$ for all $1 \leq i \leq d$\;
	$D' \gets \{z(\mathbf{x}_1),\dots,z(\mathbf{x}_N)\}$, where $z(\mathbf{x}) = \sqrt{2/d} [\cos(\langle \mathbf{w}_1, \mathbf{x} \rangle  + b_1), \ldots,  \cos(\langle \mathbf{w}_d, \mathbf{x} \rangle  + b_d )]$\;
	{\bf Clip Features:}  
	$C_s \gets \text{DPNorm}(D', \sigma_{\mathcal{C}})$ \texttt{//see Alg.~\ref{alg:dpnorm}}\;
	$\hat{D}' \gets \{\hat{z}(\mathbf{x}_1),\dots,\hat{z}(\mathbf{x}_N)\}$, where $\hat{z}(\mathbf{x}_{i}) = z(\mathbf{x}_{i}) / \max \left(1, || z(\mathbf{x}_{i}) ||_2/C_s\right)$\;
	Initialize cluster centers $\mathbf{\hat{c}}_1, \mathbf{\hat{c}}_2, \ldots, \mathbf{\hat{c}}_k$ on public data\;
	\For {$t \in [T]$}
	{
		\For {$i \in [k]$}
		{	
			\textbf{Assign:} $\hat{D}_i \gets \{\mathbf{x} : \arg\min_j ||\hat{z}(\mathbf{x}) - \mathbf{\hat{c}}_j ||_2^2 = i\}$\;
			\textbf{Update:} $\hat{n}_i \gets |\hat{D}_i| + \mathcal{N}(0,\sqrt{2}\sigma_{\mathcal{K}})$\;
			$\mathbf{\hat{c}}_i \gets 1/\hat{n}_i  \left( \sum_{\mathbf{x} \in \hat{D}_i} \hat{z}(\mathbf{x}) +  \mathcal{N}(0,\sqrt{2}C_s\sigma_{\mathcal{K}} \mathbf{I} )\right)$\;
		}
	}
	\KwOut{$\hat{D}_1, \hat{D}_2, \ldots, \hat{D}_k$} 
\end{algorithm}

Standard DP $k$-means~\cite{BlumDMN05} releases the noisy cluster centers computed iteratively using a noisy variant of  Lloyd's algorithm;  in each iteration, gaussian noise with scale $\sqrt{2}\sigma_{\mathcal{K}}$ is added to the size of all clusters, and with scale $\sqrt{2}\sigma_{\mathcal{K}}C_s$ to the sum of all cluster members in each cluster. These noisy values are used to compute the noisy cluster centers $\{\mathbf{\hat{c}}_1, \ldots, \mathbf{\hat{c}}_k\}$.\footnote{We initialize clusters centers to random records drawn from \emph{publicly available} non-sensitive data generated by the same distribution as the sensitive data.  We only need $k$ representative samples, and such public datasets already exist for images, location, and  medical data.} To determine the scale of the gaussian noise, the $L_2$-sensitivity of the cluster size and that of the sum of norms must be known within each cluster. Although the $L_2$-sensitivity of the set of cluster size is always $\sqrt{2}$ (a single record can change the size of at most 2 clusters), such \emph{a priori} bound does not exist for the $L_2$-norm of the feature vectors in general. Hence, we need to clip all feature vectors in $L_2$-norm before applying standard DP $k$-means, where the clipping threshold $C_s$  should be set to the average norm of the feature vectors (i.e., $(1/N) \sum_{\mathbf{x}\in D} || z(\mathbf{x}_i)||_2$) and is approximated by  Alg.~\ref{alg:dpnorm}. %
Replacing $z(\mathbf{x}_{i})$ with $\hat{z}(\mathbf{x}_{i}) = z(\mathbf{x}_{i}) / \max \left(1, || z(\mathbf{x}_{i}) ||_2/C_s\right)$ guarantees that all feature vectors are kept as long as their norm is less then $C_s$, or they are scaled down to have a norm of $C_s$.

Nevertheless, for kernel functions like the Radial Basis Function (RBF)\footnote{If the kernel function is RBF, i.e., \hspace*{-0.05cm}$\kappa(\mathbf{x}, \mathbf{y})\hspace*{-0.1cm}=\hspace*{-0.1cm}\exp(-\gamma ||\mathbf{x} - \mathbf{y}||_2)$, then \hspace*{-0.05cm}$p(\mathbf{w})$ has zero centered gaussian distribution with standard deviation $2\gamma\mathbf{I}$.}, a small  norm bound $C_s$ can be used (see Theorem~\ref{thm:clip_rff}). This bound is constant for any input data and feature size independently of the width $ \gamma$ of the RBF kernel. Thus, 
as opposed to  standard $k$-means  \cite{BlumDMN05}, our approach can detect linearly non-separable clusters, and, used with RBF kernel, add constant noise to feature vectors independently of their size $d$. 

\begin{theorem}
\label{thm:clip_rff} If $\kappa(\mathbf{x}, \mathbf{y}) = \exp(-\gamma ||\mathbf{x} - \mathbf{y}||_2)$, then
$\mathbb{E}[||z(\mathbf{x})||_2] \leq 1$ for any $\mathbf{x} \in \{0,1\}^{*}$ and $\gamma$, where the expectation is taken on the randomness of $z$.	
\end{theorem}

First, we introduce and prove Lemma~\ref{lem:trig_norm}.
\vspace*{-0.2cm}
\begin{lemma}
	\label{lem:trig_norm}
	Let $\mathcal{N}(0, \sigma)$ be a zero-centered normal random variable with standard deviation $\sigma$. Then:\vspace{0.1cm}
	\begin{compactenum}
		\item $
		\mathbb{E}[\cos(\mathcal{N}(0, \sigma))] = \exp(-\sigma^2/2)$ and \\ 
		$\mathbb{E}[\sin(\mathcal{N}(0, \sigma))] = 0$,
		\item $\mathbb{E}[\cos^2(\mathcal{N}(0, \sigma))] = (1+\exp(-2\sigma^2))/2$ and\\ $\mathbb{E}[\sin^2(\mathcal{N}(0, \sigma))] = (1 - \exp(-2\sigma^2))/2$
	\end{compactenum}
\end{lemma}  

\begin{proof}[Proof of Lemma~\ref{lem:trig_norm}]
	Let $\exp(j\mathcal{N}(0, \sigma))$ denote a complex random variable. It follows from the moment generating function of $\mathcal{N}(0, \sigma)$ that:
	$$
	\mathbb{E}[\exp(j\mathcal{N}(0, \sigma))] = \exp((j\sigma)^2/2)= \exp(-\sigma^2/2) 
	$$
	which means that:
	\begin{equation*}
	\begin{split}
	\mathbb{E}[ \cos(\mathcal{N}(0, \sigma)) + j\sin(\mathcal{N}(0, \sigma))]
	& = \mathbb{E}[\exp(j\mathcal{N}(0, \sigma))] \\ & = \exp(-\sigma^2/2) 		
	\end{split}
	\end{equation*}
	This implies that $\mathbb{E}[\cos(\mathcal{N}(0, \sigma))] = \exp(-\sigma^2/2)$ and  $\mathbb{E}[\sin(\mathcal{N}(0, \sigma))] = 0$ due to the linearity of expectation. 
	Hence:
	\begin{equation*}
	\begin{split}
	\mathbb{E}[\cos^2(\mathcal{N}(0, \sigma))] & = \mathbb{E}[(1 + \cos(2\mathcal{N}(0, \sigma)))/2] \\ 
	& = (1+\exp(-2\sigma^2))/2		
	\end{split}
	\end{equation*}
	and
	\begin{equation*}
	\begin{split}
	\mathbb{E}[\sin^2(\mathcal{N}(0, \sigma))] & = \mathbb{E}[(1 + \cos(2\mathcal{N}(0, \sigma)))/2] \\ & = (1-\exp(-2\sigma^2))/2 
	\end{split}
	\end{equation*}
	where we used that $2\mathcal{N}(0, \sigma) = \mathcal{N}(0, 2\sigma)$.
\end{proof}

\begin{proof}[Proof of Theorem~\ref{thm:clip_rff}]
	If $\kappa(\mathbf{x}, \mathbf{y}) = \exp(-\gamma ||\mathbf{x} - \mathbf{y}||_2)$, then  $p(\mathbf{w})= \frac{1}{2\pi} \int_{\mathbb{R}^m} \exp(-j \langle \mathbf{w}, \mathbf{x} \rangle) \kappa(\mathbf{w})d \mathbf{x}$ has zero centered gaussian distribution with standard deviation $2\gamma\mathbf{I}$.

	{\footnotesize
		\begin{align}
		& \mathbb{E}[||z(\mathbf{x})||_2] = \mathbb{E}\left[\left((2/d)\sum_{i=1}^d \cos^2(\langle \mathcal{N}(0, 2\gamma\mathbf{I}), \mathbf{x} \rangle + \mathcal{U}[0,2\pi] )\right)^{\frac{1}{2}}\right] \notag	\\	
		&\leq \sqrt{\frac{2}{d}} \left(\sum_{i=1}^d \mathbb{E}\left[\cos^2(\langle \mathcal{N}(0, 2\gamma\mathbf{I}), \mathbf{x} \rangle + \mathcal{U}[0,2\pi] )]\right]\right)^{\frac{1}{2}} \tag{by Jensen's inequality and the linearity of expectation}\\
		&\leq \sqrt{\frac{2}{d}} \left(\sum_{i=1}^d \mathbb{E}\left[\cos^2(\langle \mathcal{N}(0, 2\gamma\mathbf{I}), \mathbf{x}\rangle)/2 +\sin^2(\langle \mathcal{N}(0, 2\gamma\mathbf{I}), \mathbf{x}\rangle)/2 \right]\right)^{\frac{1}{2}}
		\notag \\
		&\leq \sqrt{\frac{1}{d}} \left(\sum_{i=1}^d \mathbb{E}\left[\cos^2(\mathcal{N}(0, 2\gamma \sqrt{|| \mathbf{x} ||_1})\right] + \mathbb{E}\left[\sin^2( \mathcal{N}(0, 2\gamma \sqrt{|| \mathbf{x} ||_1}) \right]\right)^{\frac{1}{2}}\tag{by Lemma~\ref{lem:trig_norm}} \\
		&\leq 1
		\notag 
		\end{align}
	}
	where, in the second inequality, we used that  $\cos^2(a+b) = \cos^2(a) \cos^2(b)  - 2 \cos(a) \sin(a) \cos(b) \sin(b) + \sin^2(a) \sin^2(b)$, $\mathbb{E}[\cos(\mathcal{U}[0,2\pi])] = \mathbb{E}[\sin(\mathcal{U}[0,2\pi])] = 0$,
	$\mathbb{E}[\cos^2(\mathcal{U}[0,2\pi])] = \mathbb{E}[\sin^2(\mathcal{U}[0,2\pi])] = 0.5$.
\end{proof}

\smallskip
Therefore, DP kernel $k$-means has two main advantages over standard DP $k$-means~\cite{BlumDMN05}. First, kernel $k$-means can find linearly non-separable clusters. Second, if it is used with RBF kernel, the added noise is independent of the $L_2$-norm of the data records. As we show in Section~\ref{sec:eval}, this can lead to much larger clustering accuracy especially for stringent privacy requirements (i.e., for $\varepsilon < 0.5$) even for large dimensional data.

\begin{algorithm}[t]
	\small
	\caption{DPNorm: Private Approximation of Average Norm\label{alg:dpnorm}}
	\DontPrintSemicolon
	\KwIn{\emph{Data:} $S = \{\mathbf{x}_{c_1},\dots,\mathbf{x}_{c_{|S|}}\}$, \emph{Noise scale:} $\sigma_{\mathcal{C}}$, Max. norm bound: $C_{\max}$, Max. number of discretized norm bounds: $w$}
	$C_j \gets j\cdot C_{\max}/w$ for $0 \leq j \leq w$\; 
	$C_s \gets \arg\max_{j\geq 1}\{t_j + \mathcal{N}(0,\sqrt{2}\sigma_{\mathcal{C}})\}$, where $t_j = |\{\mathbf{x} \in S:  C_{j-1} < || \mathbf{g}(\mathbf{x}) ||_2 \leq C_j \}|$\;
	\KwOut{$C_s$} 
	\vspace{0.1cm}
\end{algorithm}

\subsection{Private Stochastic Gradient Descent}
\label{sec:DPSGD}
We now present our private SGD technique, summarized in Alg.~\ref{alg:dpsgd}, considering a single SGD batch iteration.
Our starting point is the work by Abadi et al.~\cite{abadi2016deep}: similar to theirs, our solution provides differential privacy to the training data by first clipping the norm of the gradient update of each record, and then perturbing these clipped gradients by the Gaussian mechanism. However, we achieve better accuracy as the clipping threshold is selected adaptively in each SGD iteration.
In particular, in each SGD iteration, we also (1)  compute the gradient of the loss function $\mathcal{L}$ on a random subset $S$ of records (denoted as ``batch'') in Line 2 of Alg.~\ref{alg:dpsgd}, (2) clip the $L_2$ norm of the gradient of each record in $S$ to have a norm at most $C_s$ (in Line 3), (3) add gaussian noise $\mathcal{N}(0,\sqrt{2} \sigma_{\mathcal{G}} C_s \mathbf{I})$ to the average of these clipped gradient updates (Line 6), and finally (4) perform the descent step (Line 7).
At the end, the updated model parameters $\theta$ are returned.
A complete training epoch on the \emph{whole} dataset $D$ consists of $(|D|/L)$ SGD iterations, which are required to process all records in every cluster on average.
Indeed, each record in a cluster $\hat{D}_s$ is selected with probability $(|\hat{D}_s|/\sum_{i=1}^k |\hat{D}_i|) \times (L / ||\hat{D}_s|) = L / |D|$, where  $\sum_{i=1}^k |\hat{D}_i| = |D|$. Notice that the $L_2$-sensitivity of $\sum_i{\mathbf{\hat{g}}(\mathbf{x}_i)}$ is $\sqrt{2}C_s$, as the norm of every $\mathbf{\hat{g}}(\mathbf{x}_i)$ is at most $C_s$, and one record can change at most two clusters.

\begin{algorithm}[t]
	\small
	\caption{Private Stochastic Gradient Descent\label{alg:dpsgd}}
	\DontPrintSemicolon
	\KwIn{\emph{Data:} $\hat{D}$, \emph{Model parameters:} weights and biases $\theta$,  \emph{Noise scales:} $\sigma_{\mathcal{C}}$, $\sigma_{\mathcal{G}}$, Loss function: $\mathcal{L}(\theta) = \frac{1}{|\hat{D}|}\sum_i \mathcal{L}(\theta,\mathbf{x}_{c_i})$,  Learning rate: $\eta$, Batch size: $L$}
	{\bf Sampling:} Take a random sample $S = \{ \mathbf{x}_{c_1}, \ldots, \mathbf{x}_{c_{|S|}}\}$ of $\hat{D}$ with sampling probability $q=L/|\hat{D}|$\;
	{\bf Compute Gradient:} For each $\mathbf{x}_{c_i} \in S$, compute $\mathbf{g}(\mathbf{x}_{c_i}) \gets \nabla_{\theta} \mathcal{L}(\theta,\mathbf{x}_{c_i})$\;
	{\bf Clip Gradient:} 
	$S' \gets \{\mathbf{g}(\mathbf{x}_{c_i}), \ldots, \mathbf{g}(\mathbf{x}_{c_{|S|}} )\}$\;
	$C_s \gets \text{DPNorm}(S', \sigma_{\mathcal{C}})$  \texttt{//see Alg.~\ref{alg:dpnorm}}\;
	$\mathbf{\hat{g}}(\mathbf{x}_{c_i}) \gets \mathbf{g}(\mathbf{x}_{c_i}) / \max \left(1, \frac{ || \mathbf{g}(\mathbf{x}_{c_i}) ||_2}{C_s}\right)$\;
	\textbf{Add noise:} $\mathbf{\tilde{g}} \gets \frac{1}{L} \left(\sum_{i=1}^{|S|}{\mathbf{\hat{g}}(\mathbf{x}_{c_i})} + \mathcal{N}(0,\sqrt{2}\sigma_{\mathcal{G}} C_s \mathbf{I})\right)$  \;
	{\bf Descent:} $\theta \gets \theta - \eta \mathbf{\tilde{g}}$\;
	
	\KwOut{$\theta$} 
	\vspace{0.1cm}
\end{algorithm}

\subsection{Adaptive selection of the norm bound} 
\label{sec:dpnorm}
Both our private SGD method (in Line 4 of Alg.~\ref{alg:dpsgd}) and private kernel $k$-means (in Line 4 of Alg.~\ref{alg:dpkkmeans}) require the differentially private computation of the average $L_2$-norm in a given set of records, which is then used as the clipping threshold $C_s$ in both algorithms. For this purpose, these algorithms invoke DPNorm which is detailed in Alg.~\ref{alg:dpnorm}.
In fact, our SGD technique
differs from the original private SGD method~\cite{abadi2016deep} in the selection of the norm bound $C_s$ (in Line 3-5 of Alg.~\ref{alg:dpsgd}).
In the original approach~\cite{abadi2016deep}, $C_s$ is provided as input to the private SGD and no guideline is given  how to compute its value without violating differential privacy. Moreover,  the selection of the norm bound $C_s$ has a large impact on the performance of the private SGD in general. If $C_s$ is too small, there will be slow convergence. Conversely, if it is too large, unnecessarily large gaussian noise will be introduced on the gradient update. Intuitively, $C_s$ should be adjusted so that $ || \mathbf{g}(\mathbf{x}_{c_i})||_2 \approx C_s$ for each record $\mathbf{x}_{c_i}$. This guarantees that the contribution of $\mathbf{x}_{c_i}$ to $\mathbf{\tilde{g}}$ is maximally preserved with the smallest relative error.
Hence, instead of fixing $C_s$ for the whole training, we aim to compute $C_s$ adaptively for each batch as $C_s = (1/L) \sum_i  ||\mathbf{g}(\mathbf{x}_{c_i})||_2$. 
This adaptive approach would ensure fast convergence with small error, and also adapt to the gradient update of every batch. Indeed, SGD is iterative, so the gradient update $\mathbf{\tilde{g}}$ of a batch/iteration depends on that of the previous batch/iteration, which means that $(1/L) \sum_i ||\mathbf{g}(\mathbf{x}_{c_i})||_2$ is different for each batch. 

In DPNorm (see Alg.~\ref{alg:dpkkmeans}), the computation of the average norm in a set $S$ of records is randomized to guarantee privacy. A naive solution is to add Gaussian noise to this average, i.e., $C_s = (1/|S|) \sum_{\mathbf{x} \in S}  ||\mathbf{x}||_2 + \mathcal{N}(0, s \cdot \sigma'/L)$, where $s \geq  \max_{\mathbf{x} \in S} ||\mathbf{x}||_2$. However, $\max_{\mathbf{x} \in S} ||\mathbf{x}||_2$ is data-dependent and can be too large if there are outliers in $S$. Instead, we approximate $C_s$ such that its value is close to the norm of many records in $S$, i.e., it is a good approximator of $(1/L) \sum_{\mathbf{x} \in S}  ||\mathbf{x}||_2  $.  %
In particular, we discretize the domain of $C_s$ by dividing $(0,C_{\max})$ uniformly into $w$ intervals (Line 1 of Alg.~\ref{alg:dpnorm}).
Then, we use the Gaussian mechanism (Line 2 of Alg.~\ref{alg:dpnorm}) to select among the upper bounds $C_j = jC_{\max}/w$ of these intervals $(0\leq j \leq w)$, which will be the norm bound $C_s$ for  $S$. Specifically, we build a histogram where bin $i$ equals the number of records whose gradient norm  falls within $(C_{i-1}, C_{i}]$. Then, the (noisy) mode of this histogram is computed by adding independent gaussian noise $\mathcal{N}(0,\sqrt{2}\sigma_{\mathcal{C}})$ to each count, and selecting the bin which has the greatest noisy count.
Note that the $L_2$-sensitivity of the histogram is always bounded by $\sqrt{2}$ no matter how large $\max_{\mathbf{x} \in S} ||\mathbf{x}||_2$ is. 

\subsection{Synthetic data generation}
To generate an accurate synthetic dataset, data generation should mimic the training process; in order to generate a synthetic sample, a model with parameter $\theta_i$ is first selected randomly with probability $1/|\hat{D}_{i}|$, then a synthetic sample is generated using the selected model. This process is repeated until $|D|$ samples are obtained.

The above generation process ensures that low quality models which were not selected in training are also less likely to be used for data generation. 
In particular, though each model is trained during the same number of epochs \emph{on its own cluster} in expectation, there is no guarantee that all $k$ models will produce identical quality of synthetic samples due to randomization. Indeed, a cluster can potentially contain dissimilar samples, or it may be too small to be selected in Line 4 of Alg.~\ref{alg:dpsgd} and hence fail to converge.

\section{Privacy Analysis} 
\label{sec:analysis}
In this section, we present the formal privacy analysis of DPGM.
Recall that DPGM is the composition of private kernel $k$-means and private SGD.
Let $\mathcal{K}$ denote the private kernel $k$-means algorithm whose output is the noisy mapped cluster centers after $T_{\mathcal{K}}$ clustering iterations (i.e., $\mathcal{K}(D) =  \{\mathbf{\hat{c}}_1, \ldots, \mathbf{\hat{c}}_k\}$). $\mathcal{K}$ is composed of (1) selecting the norm bound using DPNorm and (2) $T_{\mathcal{K}}$ iterations of $k$-means.
Let $\mathcal{G}_1$ denote the gaussian mechanism which selects the norm bound as per Section~\ref{sec:dpnorm}.
A single $k$-means iteration is the 2-fold adaptive composition of two gaussian mechanisms $\mathcal{G}_2$ and $\mathcal{G}_3$ (in Line 10-11 of Alg.~\ref{alg:dpkkmeans}), where $\mathcal{G}_2$ perturbs the cluster size (Line 10), while $\mathcal{G}_3$ adds noise to the sum of Fourier features of the cluster members (Line 11). The $L_2$-sensitivity of the size of every clusters is $\sqrt{2}$, as changing a single record can change the size of at most two clusters. Similarly, the $L_2$-sensitivity of the sum of Fourier features of the cluster members is $\sqrt{2}C_s$ as it is detailed in Section~\ref{sec:DPKMEANS}. 

Since $\mathcal{K}$ is the $T_{\mathcal{K}}$-fold adaptive composition of $T_{\mathcal{K}}$ clustering iterations, it follows from Theorem~\ref{thm:comp} and Lemma~\ref{lem:gauss_alpha}:  
{\small
\begin{align}
\alpha_{\mathcal{K}}(\lambda) &\leq T_{\mathcal{K}}(\alpha_{\mathcal{G}_1}(\lambda)  + \alpha_{\mathcal{G}_2}(\lambda) + \alpha_{\mathcal{G}_3}(\lambda)) 
\notag\\
&\leq T_{\mathcal{K}}(\lambda^2 + \lambda)(1/4\sigma_{\mathcal{C}}^2  + 1/2\sigma_{\mathcal{K}}^2)  
\label{eq:ak}
\end{align}}
Note that if the RBF kernel is used in  kernel $k$-means (i.e., $\kappa(\mathbf{x}, \mathbf{y}) = \exp(-\gamma ||\mathbf{x} - \mathbf{y}||_2)$ in Alg.~\ref{alg:dpkkmeans}), then $\alpha_{\mathcal{G}_1}(\lambda) =0$ and $\alpha_{\mathcal{K}}(\lambda) \leq T_{\mathcal{K}}(\lambda^2 + \lambda)/2\sigma_{\mathcal{K}}^2$ since $C_s=1$ is a priori bound on the $L_2$-norm of every feature vector (cf.~Theorem~\ref{thm:clip_rff}). 

Let $\mathcal{S}_k$ denote the private SGD algorithm whose output is the noisy model parameters after $T_{\mathcal{S}}$ SGD iterations (i.e., $\mathcal{S}(D) =  \{\theta_1, \ldots, \theta_k\}$, computed in the last iteration of Alg.~\ref{alg:dpkkmeans}), and the input is the cluster centers $\{\mathbf{\hat{c}}_1, \ldots, \mathbf{\hat{c}}_k\}$ provided by $\mathcal{K}$.  At the very beginning, $\mathcal{S}$ assigns each record to its closest cluster center in feature space to obtain $k$ non-overlapping training sets (this is implemented by the last iteration in Alg.~\ref{alg:dpkkmeans}). Changing a single record alters at most a single record in at most 2 training sets (clusters), as the modified record can be moved from one to another training set. 
Since all training sets are non-overlapping, each record is selected in an SGD iteration with probability $q= (|\hat{D}_s|/|D|) \times (L / ||\hat{D}_s|) = L/|D|$ for any $k$. Moreover, each of the $k$ models are trained independently,  so $\alpha_{\mathcal{S}_k}(\lambda) \leq \alpha_{\mathcal{S}_1}(\lambda)$, where  $\mathcal{S}_1$ denotes the case when $k=1$ (i.e., a single model is trained on the whole dataset $D$ during $T_{\mathcal{S}}$ epochs).

The complete SGD training of $\mathcal{S}_1$ (Line 3-6 in Alg.~\ref{alg:ours}) is the $T_{\mathcal{S}}$-fold adaptive composition of $T_{\mathcal{S}}$ SGD iterations, where we jointly use two perturbation mechanisms $\mathcal{G}_4$ and $\mathcal{G}_5$ in each iteration;  $\mathcal{G}_4$  selects a batch uniformly at random and computes the norm bound $C_s$ (Line 4 of Alg.~\ref{alg:dpsgd}) for this batch,  then $\mathcal{G}_5$ selects the same batch and perturbs its gradient updates with gaussian noise whose magnitude is calibrated to $C_s$ (Line 6 of Alg.~\ref{alg:dpsgd}).  The composition of these two mechanisms uses independent source of randomness through different SGD iterations, hence we can use Theorem~\ref{thm:comp} to quantify the overall privacy. However, within a single iteration, $\mathcal{G}_4$  and $\mathcal{G}_5$  do not use independent source of randomness, although both mechanisms use independent gaussian noise but select the same batch $S$ from the dataset. The following theorem computes $\alpha_{\mathcal{S}_1}(\lambda)$, and is a generalization of Theorem~\ref{thm:comp} when the component mechanisms can use dependent source of randomness. 

\begin{theorem}[General Moments Accountant]
	\label{thm:dep_comp}	
		Let $\alpha_{\mathcal{A}_i}(\lambda) $ be $ \max_{D,D'} \log\mathbb{E}_{O \sim \mathcal{A}(D)}[\exp(\lambda \mathcal{P}(\mathcal{A},D,D',O))]$, and $\mathcal{A}_{1:k}$ be the $k$-fold adaptive composition of $\mathcal{A}_1, \mathcal{A}_2, \ldots, \mathcal{A}_{k}$. Then:  
		\begin{compactenum}
			\item  $\alpha_{\mathcal{A}_{1:k}}(\lambda) \leq  \sum_{i=1}^k j_i \alpha_{\mathcal{A}_i}(\lambda/j_i)$
			\item $\mathcal{A}_{1:k}$ is $(\varepsilon, \min_{\lambda}\exp(\sum_{i=1}^kj_i\cdot\alpha_{\mathcal{A}_i}(\lambda/j_i) - \lambda \varepsilon))$-DP 
		\end{compactenum}
	for any $\sum_{i=1}^k j_i = 1$, where $j_i > 0$ and $\mathcal{A}_1, \mathcal{A}_2, \ldots, \mathcal{A}_{k}$ can use dependent coin tosses.  
\vspace*{-0.3cm}
\end{theorem}
\begin{proof}
We adapt the proof of Theorem 2 in~\cite{abadi2016deep} to the case when the composite mechanism $\mathcal{A}_{1:k}$ consists of dependent mechanisms.  
Here we detail the complete proof for the sake of clarity.

Let $\mathcal{A}_{1:k}$ denote the composition of $\mathcal{A}_1, \mathcal{A}_2, \ldots, \mathcal{A}_k$ and $O = (O_1, O_2, \ldots, O_k)$.
Recall that, from Definition~\ref{def:DP},  $\mathcal{A}_{1:k}$ is $(\varepsilon, \delta)$-DP, if $\Pr_{O \sim \mathcal{A}_{1:k}(D)}[\mathcal{P}(\mathcal{A}_{1:k},D,D',O) > \varepsilon] \leq \delta$. Then:

\begin{figure*}[!t]
\centering
\begin{subfigure}[t]{0.335\textwidth}
\includegraphics[width=1\textwidth]{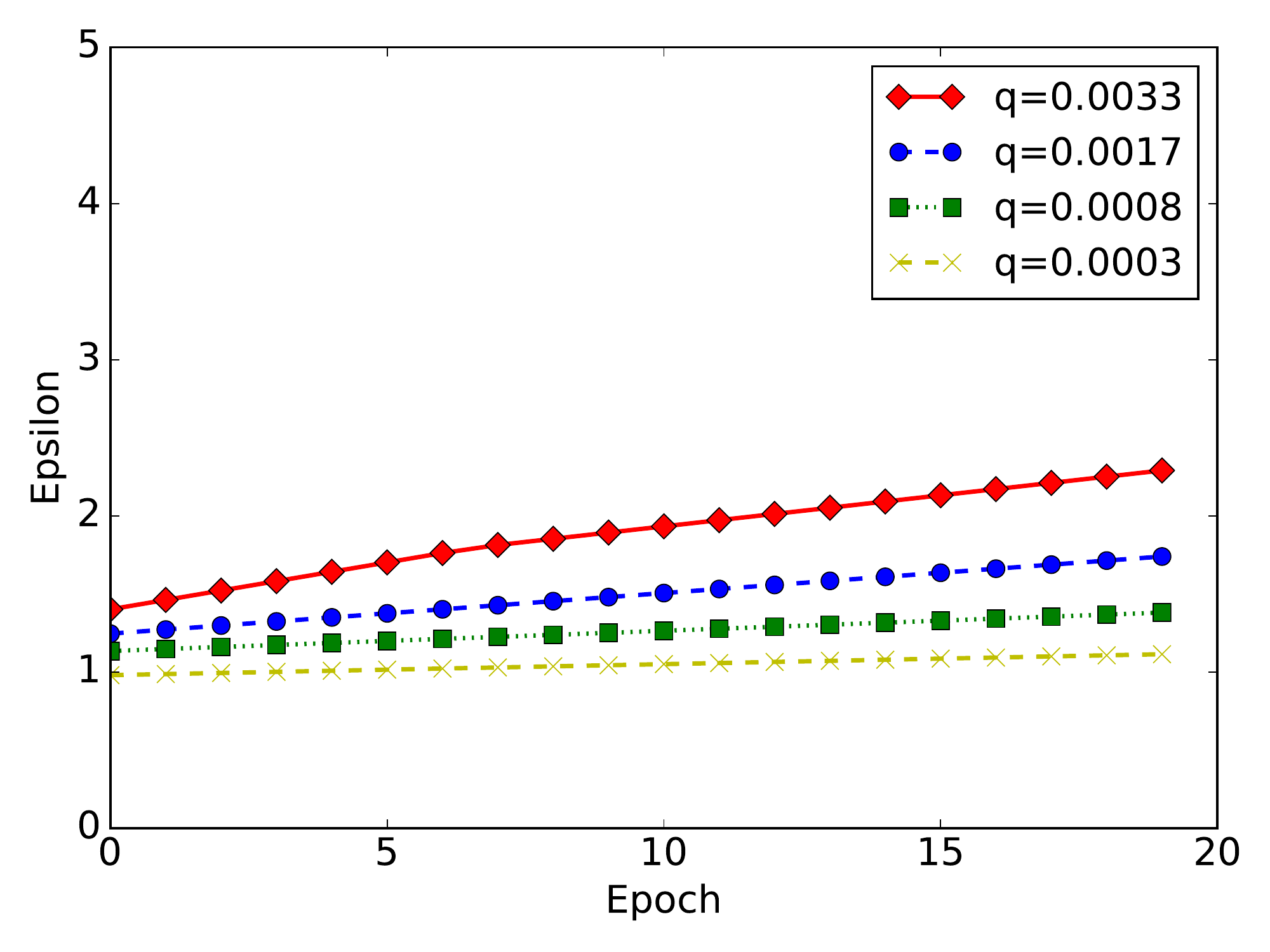}
\caption{\footnotesize Sampling probability $q$ ($\sigma_{\mathcal{G}}=1.0$, $\sigma_{\mathcal{K}}=40.0$)}
\label{fig:q-epsilon}
\end{subfigure}
\hspace*{-0.15cm}
\begin{subfigure}[t]{0.333\textwidth}
\includegraphics[width=1\textwidth]{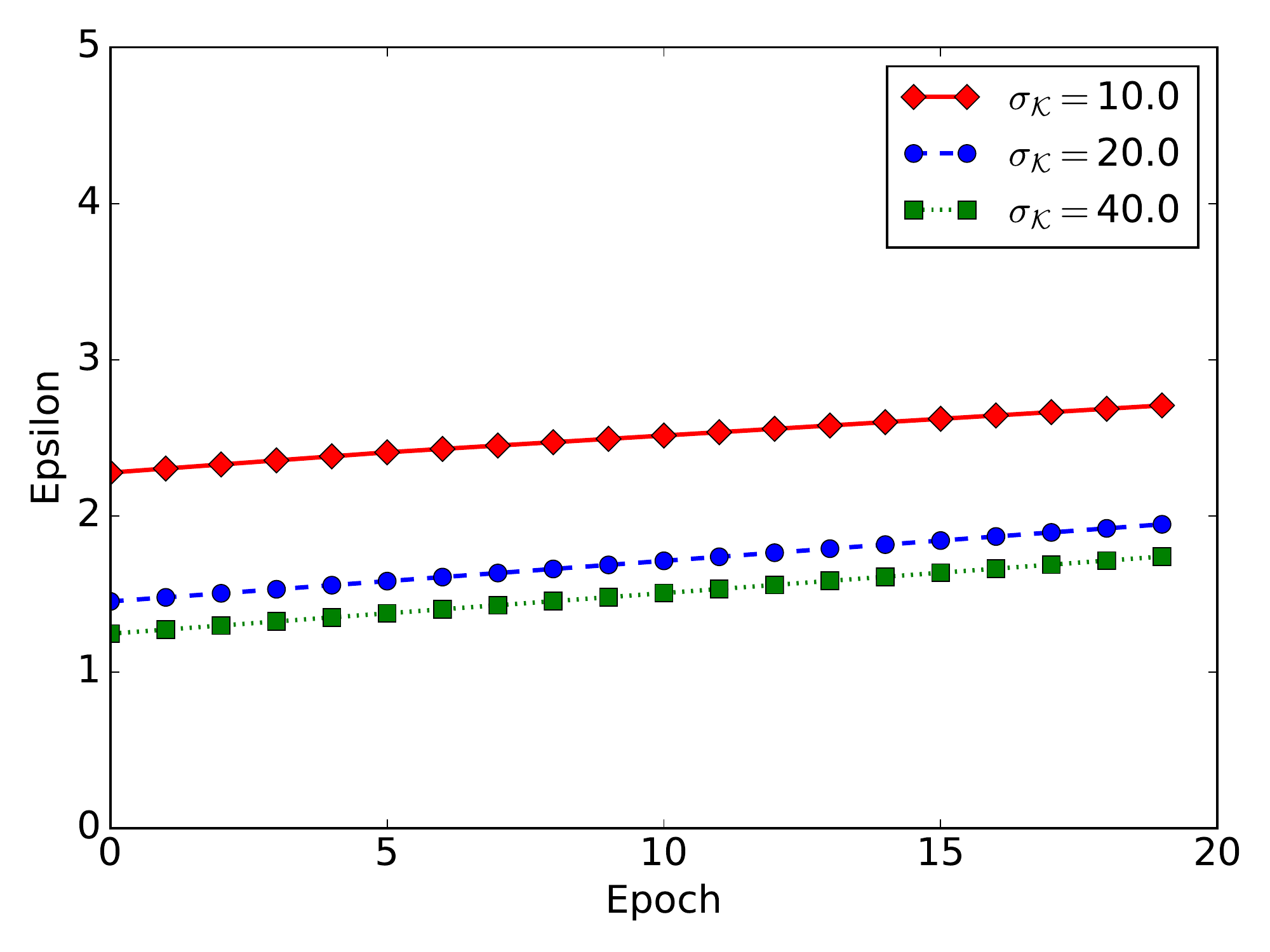}
\caption{\footnotesize Clustering noise $\sigma_{\mathcal{K}}$ ($\sigma_{\mathcal{G}}=1.0$, $q = 0.0017$)}
\label{fig:sigma1-epsilon}
\end{subfigure}
\hspace*{-0.3cm}
\begin{subfigure}[t]{0.333\textwidth}
\includegraphics[width=1\textwidth]{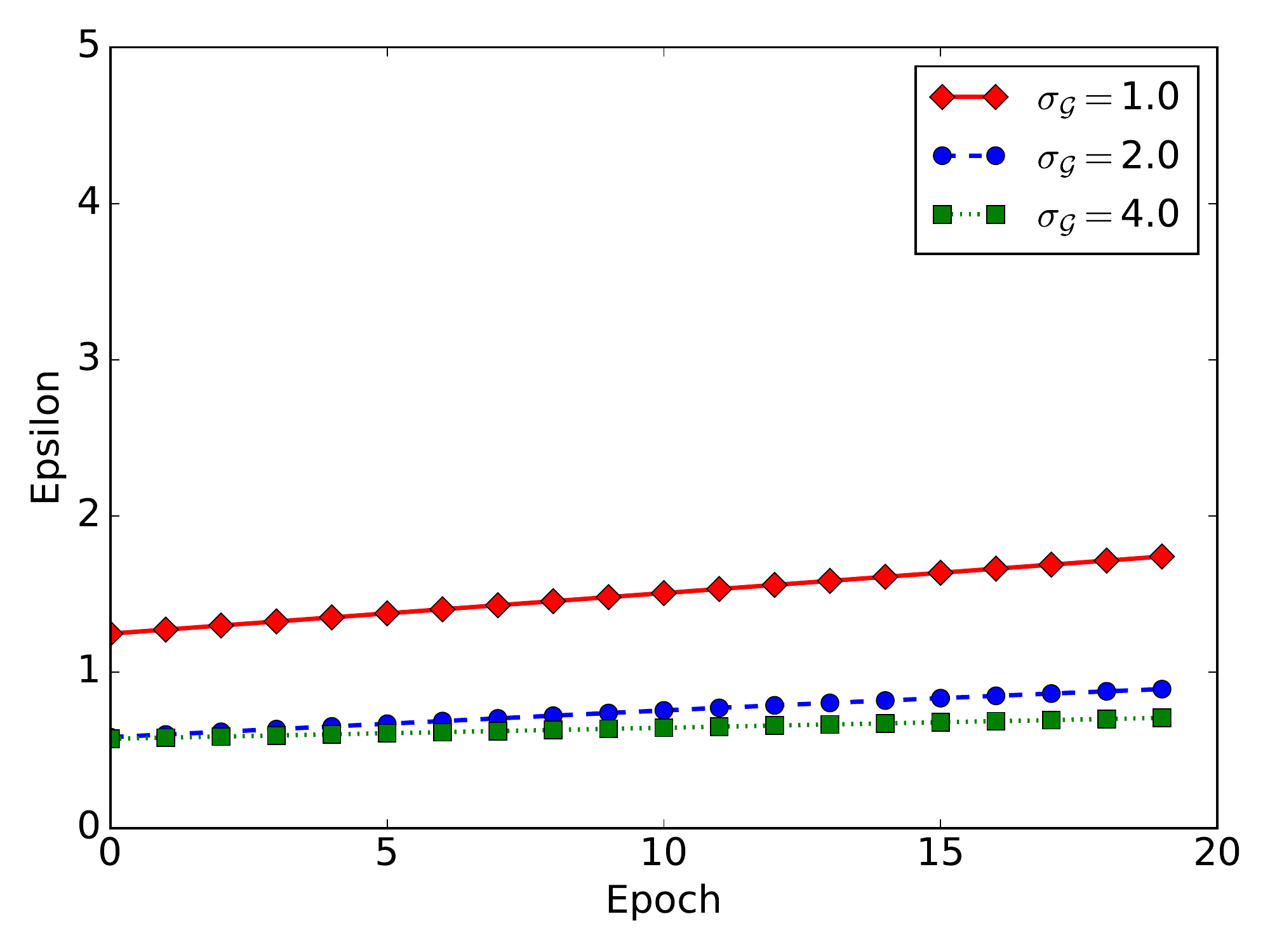}
\caption{\footnotesize SGD noise $\sigma_{\mathcal{G}}$ ($\sigma_{\mathcal{K}}=40.0$, $q = 0.0017$)} \label{fig:sigma4-epsilon}
\end{subfigure}
\vspace{-0.1cm}
\caption{$\varepsilon$ value as a function of the number of SGD training epochs for MNIST ($\delta=10^{-5}, T_{\mathcal{K}} = 20$)}
\label{fig:epsilon}
\end{figure*}

{\footnotesize
\begin{align}
\mathcal{P}&(\mathcal{A}_{1:k},D,D',O) \notag \\
& = \log\frac{\Pr[\mathcal{A}_{1:k}(D) = O]}{\Pr[\mathcal{A}_{1:k}(D') = O]} \notag \\
&=\log\prod_{i=1}^{k} \frac{\Pr[\mathcal{A}_{i}(D) = O_i | \mathcal{A}_{i-1}(D) = O_{i-1}, \ldots,  \mathcal{A}_{1}(D) = O_{1}  ]}{\Pr[\mathcal{A}_{i}(D') = O_i | \mathcal{A}_{i-1}(D') = O_{i-1}, \ldots,  \mathcal{A}_{1}(D') = O_{1}  ]} \tag{by the Chain rule}\\
&=\sum_{i=1}^{k} \log\frac{\Pr[\mathcal{A}_{i}(D) = O_i | \mathcal{A}_{i-1}(D) = O_{i-1}, \ldots,  \mathcal{A}_{1}(D) = O_{1}  ]}{\Pr[\mathcal{A}_{i}(D') = O_i | \mathcal{A}_{i-1}(D') = O_{i-1}, \ldots,  \mathcal{A}_{1}(D') = O_{1} ]} \notag \\
&= \sum_{i=1}^k \mathcal{P}(\mathcal{A}_{i},D,D',O_i) \label{eq:m1}
\end{align}
}%
for any neighboring datasets $D$ and $D'$. Hence,
{\footnotesize
\begin{align}
\alpha_{\mathcal{A}_{1:k}}(\lambda) &= \max_{D,D'} \log \mathbb{E}_{O \sim \mathcal{A}(D)}[\exp(\lambda \mathcal{P}(\mathcal{A}_{1:k},D,D',O))] \notag \\
&= \max_{D,D'} \log \mathbb{E}_{O \sim \mathcal{A}(D)}\left[\exp\left(\lambda \sum_{i=1}^k \mathcal{P}(\mathcal{A}_{i},D,D',O_i) \right)\right] \tag{by Eq.~\eqref{eq:m1}} \\
&= \max_{D,D'} \log \mathbb{E}_{O \sim \mathcal{A}(D)}\left[\prod_{i=1}^k\exp\left(\lambda  \mathcal{P}(\mathcal{A}_{i},D,D',O_i) \right)\right] \notag\\
&\leq \max_{D,D'} \log  \prod_{i=1}^k \left( \mathbb{E}_{O_i \sim \mathcal{A}_i(D)}\left[\exp\left(\lambda  \mathcal{P}(\mathcal{A}_{i},D,D',O_i)/j_i\right)\right]\right)^{j_i}  \tag{by the generalization of H\"older's inequality}\\
&\leq \max_{D,D'} \sum_{i=1}^k j_i \log \left( \mathbb{E}_{O_i \sim \mathcal{A}_i(D)}\left[\exp\left(\lambda  \mathcal{P}(\mathcal{A}_{i},D,D',O_i)/j_i\right)\right]\right)  \notag \\
&\leq\sum_{i=1}^k j_i  \max_{D,D'}  \log \left( \mathbb{E}_{O_i \sim \mathcal{A}_i(D)}\left[\exp\left(\lambda  \mathcal{P}(\mathcal{A}_{i},D,D',O_i)/j_i\right)\right]\right) \notag \\
&\leq \sum_{i=1}^k j_i \alpha_{\mathcal{A}_{i}}(\lambda/j_i) \label{eq:m2}
\end{align}
}
where we can apply the generalization of H\"older's inequality in the first inequality due to the fact that $\exp(\cdot)$ is always positive.
Therefore, 
{\footnotesize
\begin{align*}
& \Pr[\mathcal{P}(\mathcal{A}_{1:k},D,D',O) \geq \varepsilon] = \Pr[\exp(\lambda\mathcal{P}(\mathcal{A}_{1:k},D,D',O))\geq \exp(\lambda\varepsilon)] \\
&\leq \mathbb{E}_{O \sim \mathcal{A}(D)}[\exp(\lambda\mathcal{P}(\mathcal{A}_{1:k},D,D',O))]/\exp(\lambda\varepsilon) \tag{by Markov's inequality} \notag\\
&\leq \exp(\alpha_{\mathcal{A}_{1:k}}(\lambda)  - \lambda \varepsilon) \leq \exp\left(\sum_{i=1}^k j_i \alpha_{\mathcal{A}_{i}}(\lambda/j_i)  - \lambda \varepsilon\right)   \tag{by Eq.~\eqref{eq:m2}} 
\end{align*}
}
The claim follows from Definition~\ref{def:DP}.
\end{proof}
\noindent Therefore, it follows from  Theorem~\ref{thm:comp} and~\ref{thm:dep_comp} that:
{\small 
\begin{align} 
\footnotesize
\alpha_{\mathcal{S}_k}(\lambda) &\leq \alpha_{\mathcal{S}_1}(\lambda) \notag \\
&\leq T_{\mathcal{S}} \cdot \min_{j_1,j_2 \in (0,1):  j_1 + j_2 = 1} \left( j_1 \alpha_{\mathcal{G}_4}(\lambda/j_1) + j_2 \alpha_{\mathcal{G}_5}(\lambda/j_2)\right)		
\label{eq:as}
\end{align}}

We compute  $\alpha_{\mathcal{G}_4}(\lambda)$ and $\alpha_{\mathcal{G}_5}(\lambda)$ similarly to \cite{abadi2016deep}. That is, let $\mu_0(x|\sigma) = g(x|\sigma)$ and $\mu_1(x|\sigma) = (1-q) g(x|\sigma) + q g(x-1|\sigma)$, where $q = L / |D|$ is the probability that a record is included in the batch $S$ of an SGD iteration and $g(x|\sigma) = \frac{1}{\sqrt{2\pi\sigma^2}}e^{-x^2/2\sigma^2}$. Then, it holds:
\begin{align*}
\alpha_{\mathcal{G}_3}(\lambda) &= \log\max(E_1(\lambda, \sigma_{\mathcal{C}}), E_2(\lambda, \sigma_{\mathcal{C}})) \\
\alpha_{\mathcal{G}_4}(\lambda) &= \log\max(E_1(\lambda, \sigma_{\mathcal{G}}), E_2(\lambda, \sigma_{\mathcal{G}})) 
\end{align*}
where
{\small
\begin{align*}
E_1(\lambda,  \sigma) &=  \int_{-\infty}^{\infty}\mu_0(x|\sigma) \cdot \left(\frac{\mu_0(x|\sigma)}{\mu_1(x|\sigma)}\right)^{\lambda}dx \\
E_2(\lambda, \sigma) &= \int_{-\infty}^{\infty}\mu_1(x|\sigma) \cdot \left(\frac{\mu_1(x|\sigma)}{\mu_0(x|\sigma)}\right)^{\lambda}dx  
\end{align*}}
\hspace{0.3cm} The next theorem immediately follows from Theorem~\ref{thm:comp} and Theorem~\ref{thm:dep_comp}.

\smallskip
\begin{theorem}
	\label{thm:privacy}
	Our differentially private generative model (DPGM) is $(\min_{\lambda}\left(\alpha_{\mathcal{K}}(\lambda) + \alpha_{\mathcal{S}_k}(\lambda) -\log\delta\right)/\lambda, \delta)$-differentially private for any fixed $\delta$, where $\alpha_{\mathcal{K}}(\lambda)$  and $\alpha_{\mathcal{S}_k}(\lambda)$ are defined in Eq.~\ref{eq:ak} and~\ref{eq:as}.
\end{theorem}

\smallskip
In this paper, we use the convention that $\delta = 1 / |D|$, and compute $\varepsilon$ numerically.
Specifically,  $\varepsilon = \min_{\lambda}\left(\alpha_{\mathcal{K}}(\lambda) + \alpha_{\mathcal{S}_k}(\lambda) -\log\delta\right)/\lambda$ is minimized over integer values of $\lambda$, where $\lambda$ is usually no more than 100 in practice.  
The computation of 
$\alpha_{\mathcal{G}_3}$ and $\alpha_{\mathcal{G}_4}$  are performed through numerical integration, and it suffices to consider 10 different values of $j_1$ and $j_2$ in order to have a sufficiently small value of  $j_1\alpha_{\mathcal{G}_3}(\lambda/\ell_1) + j_2\alpha_{\mathcal{G}_4}(\lambda/\ell_2)$ in Eq.~\ref{eq:as}.
Therefore, in practice, given $\delta$, an accurate approximation of $\varepsilon$ can be obtained with negligible overhead.

\section{Experimental Evaluation}
\label{sec:eval}
In this section, we report the results of an experimental evaluation geared to compute the exact privacy guarantees of DPGM (presented in Alg.~\ref{alg:ours}). We also analyze its performance in terms of the quality of generated samples as well as counting (linear) queries computed on the synthetic data. Counting queries provide the basis of many data analysis and learning algorithms (see~\cite{BlumDMN05} for examples). Finally, we measure the accuracy of our private kernel $k$-means described in Alg.~\ref{alg:dpkkmeans}.

\begin{table}[t]
\centering
\small
\begin{tabular}{|l|r|r|r|r|}
\hline
{\textbf{\em Dataset}} & \multicolumn{1}{|c|}{$|D|$} & \multicolumn{1}{|c|}{$|\mathbb{I}|=m$} & \multicolumn{1}{|c|}{$\max ||\mathbf{x}||_1$} & \multicolumn{1}{|c|}{$\mathrm{avg}$ $||\mathbf{x}||_1$} \\ \hline
{\bf MNIST} & 60,000 &  784 & 311.69 & 102.44 \\ \hline
{\bf CDR} & 4,427,486 & 1303 & 422 & 11.42 \\ \hline
{\bf \transit} & 1,200,000 & 342 & 57 & 5.26 \\ \hline 
\end{tabular}
\vspace{-0.1cm}
\caption{The datasets used in our experiments: MNIST (images), CDR (call detail records), and TRANSIT (transport records).}
\label{tab:datasets}
\end{table}

\subsection{Experimental Setup}

\descr{Datasets.} 
We use three datasets for our evaluations, summarized in Table~\ref{tab:datasets}. MNIST is a public image dataset~\cite{lecun1998gradient}, which includes $28\times 28$-pixel images of hand-written digits, a total of $60,000$ samples. We vectorize and binarize each image to have binary data records with size $m= 784$. Throughout our experiments, we assume that each of the $60,000$ records originates from a different person.
We also use an anonymized CDR (Call Detail Record) dataset provided to us by a cell phone operator. For this dataset, $\mathbb{I}$ represents the set of cell towers of the operator in a large city with $|D|=4,427,486$ customers.  We use a simplified version of the dataset, which contains the set of visited cell towers per customer within the administrative region of the city over $128.1$ km\textsuperscript{2}, where the total number of towers is $m=1,303$. The average number of individuals per tower over this period was $38,817$ with a standard deviation of $50,911$.  

Finally, we experiment with a transit dataset, which we denote as \transit in the rest of the paper.\footnote{Note that experiments using this dataset do not appear in the ICDM'17 version of the paper.} 
Due to non-disclosure agreement, we are unable to provide specific details about the dataset,
however, we can report that the \transit dataset include the transit history of passengers in the network (with $|D| = 1,200,000$); here, $\mathbb{I}$ represents the set of $m=342$ stations in a public transportation network.

\descr{Experimental Settings.}
For RBM, we set the number of hidden units to $200$ and the learning rate is $0.01$. The biases $\mathbf{b}$ and $\mathbf{c}$ are initialized to zeros, while the initial values of the weights $\mathbf{W}$ are randomly chosen from a zero-mean Gaussian with a standard deviation of $0.01$.
For VAE, the number of hidden units is set to $200$ with single layer encoder and decoder, and a bi-dimensional latent space. We also used the rectifier activation function (ReLu) for all neurons and the Adam optimizer~\cite{kingma2014adam}.
For our purposes, it is enough to compute $\alpha(\lambda)$ for $\lambda \leq 32$. %
We set the number of the private k-means iterations to 20 and $\delta=1/|D|$. We also set $ C_{\max}=10$, $w=100$ (in Alg.~\ref{alg:dpnorm}), as different values of these parameters do not have a strong impact on the results.

We implement DPGM with both RBM (in C++) and VAE (in Python). %
Experiments are performed on a workstation running Ubuntu Server 16.04 LTS, with a 3.4 GHz CPU i7-6800K, 32GB RAM, and NVIDIA Titan X GPU card. Source code is available upon request. %

\subsection{Results with Image Dataset}

\begin{figure}[t]
 	\centering
 	\includegraphics[width=0.37\textwidth]{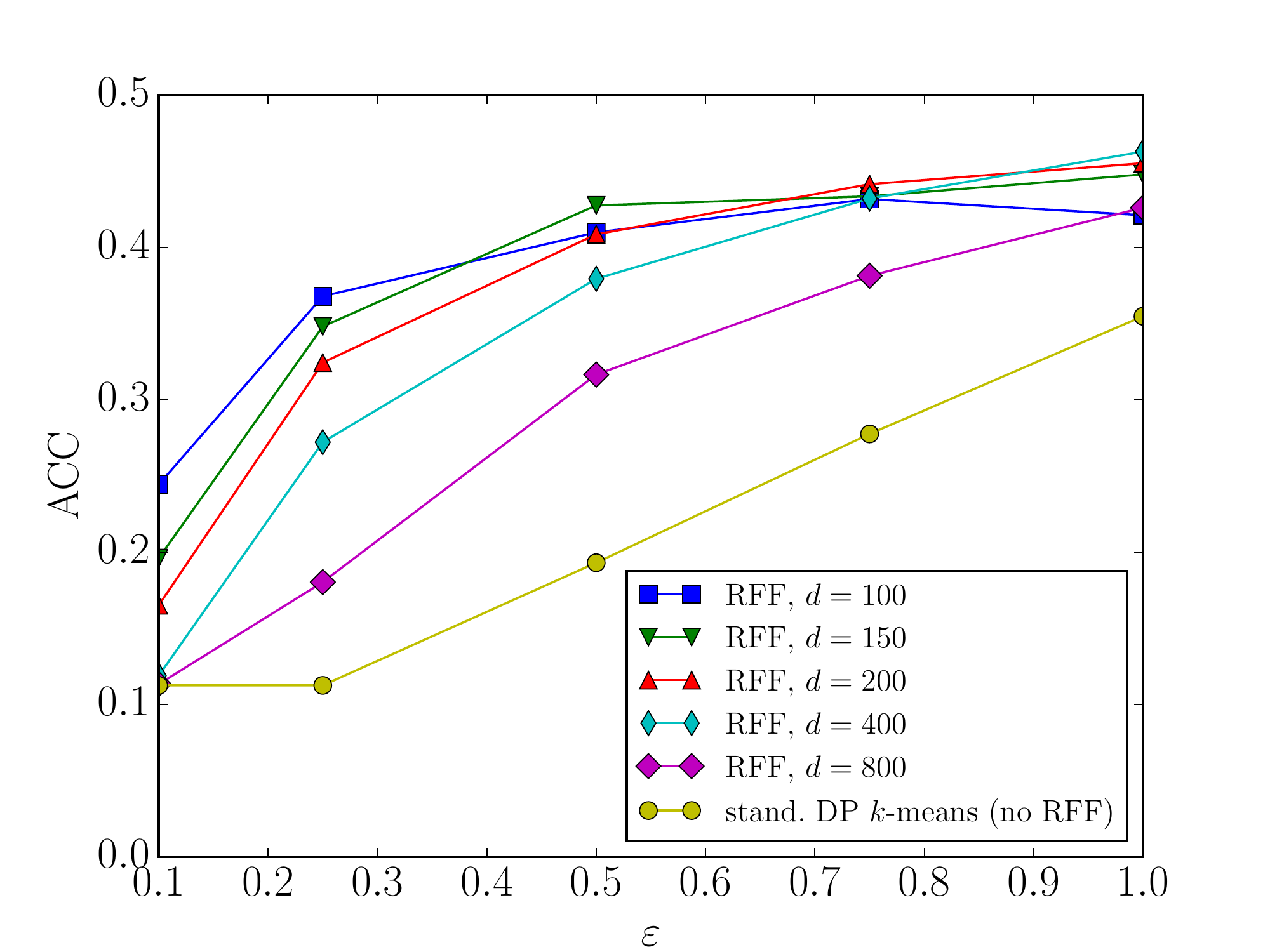}
 	\caption{Clustering accuracy as a function of $\varepsilon$ on MNIST ($\delta=10^{-5}, T_{\mathcal{K}} = 20$).}
 	\label{fig:cluster}
	\vspace{-0.1cm}
 \end{figure}

\begin{figure*}[!t]
\centering
\begin{subfigure}[t]{0.238\textwidth}
\includegraphics[width=1\textwidth]{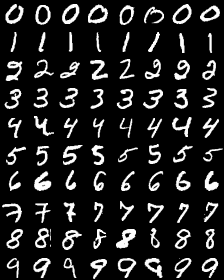}
\caption{\footnotesize Real samples}
\label{fig:real_samples}
\end{subfigure}
\begin{subfigure}[t]{0.238\textwidth}
	\includegraphics[width=1\textwidth]{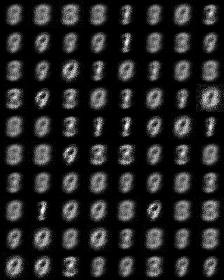}
	\caption{\footnotesize VAE w/o clustering}
	\label{fig:vae_no_cluster}
\end{subfigure}
\centering
\begin{subfigure}[t]{0.238\textwidth}
	\includegraphics[width=1\textwidth]{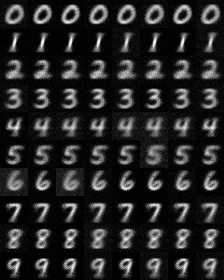}
	\caption{\footnotesize VAE with clustering}
	\label{fig:vae_samples}
\end{subfigure}
\centering
\begin{subfigure}[t]{0.238\textwidth}
\includegraphics[width=1\textwidth]{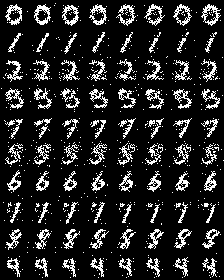}
\caption{\footnotesize RBM with clustering}
\label{fig:rbm_samples}
\end{subfigure}
\caption{Real MNIST samples and samples generated from DPGM with RBM and VAE after 20 epochs ($\varepsilon=1.74, T_{\mathcal{K}} = 20$). In \textbf{(c)} and \textbf{(d)}, each row contains 8 samples generated from a cluster.}
\label{fig:samples}
\vspace{-0.1cm}
\end{figure*}

\descr{Privacy guarantees.}
We report the privacy loss $\varepsilon$ of DPGM (Alg.~\ref{alg:ours}) in \figurename~\ref{fig:epsilon} for the MNIST dataset. Recall that $\varepsilon$ is computed from the noise level $\sigma_{\mathcal{C}}$, $\sigma_{\mathcal{K}}$, and $ \sigma_{\mathcal{G}}$, the sampling probability $q$, the number of $k$-means iterations $T_{\mathcal{K}}$, and the number of SGD iterations $T_\mathcal{\mathcal{S}}$ using Theorem \ref{thm:privacy}. \figurename~\ref{fig:epsilon} shows $\varepsilon$ depending on the number of SGD training  epochs, where one epoch consists of $\lceil 1/q \rceil$ SGD iterations. In \figurename~\ref{fig:q-epsilon}--\ref{fig:sigma4-epsilon}, we fix $\sigma_{\mathcal{C}} = 4.0$, and report the value of $\varepsilon$ as a function of the number of epochs. 
We note that larger sampling probabilities ($q$) and more epochs yield larger values of $\varepsilon$, i.e., worse privacy guarantee. 
\figurename~\ref{fig:sigma1-epsilon}--\ref{fig:sigma4-epsilon} show that larger values of $\sigma_{\mathcal{K}}$ and $\sigma_{\mathcal{G}}$ yield stronger privacy guarantees. 

\descr{Clustering accuracy.} Next, in \figurename~\ref{fig:cluster}, we compare
the private kernel $k$-means (Alg.~\ref{alg:dpkkmeans}) with RBF kernel with standard DP $k$-means \cite{BlumDMN05}. We evaluate the unsupervised clustering accuracy (ACC) \cite{XieGF16}, where 
$\mathrm{ACC} = \max_{u}\frac{|\{ \mathbf{x} : \mathbf{x} \in D \wedge \mathit{label}(\mathbf{x}) = u(\mathcal{K}(\mathbf{x}))\}|}{|D|}$, 
$\mathit{label}(\mathbf{x})$  is the ground-truth label of sample $\mathbf{x}$\footnote{For MNIST, these are digits ranging from 0 to 9.},  $\mathcal{K}(\mathbf{x})$ is the cluster assignment obtained by clustering algorithm $\mathcal{K}$, and $u$ is a one-to-one mapping between cluster assignments and labels. The best mapping can be obtained using the Hungarian algorithm.
 To make a fair comparison, we fix $C_s$ to $\sqrt{m}=28$ for standard private $k$-means without RFF features, and $C_s=1$ for private kernel $k$-means with RFF features  based on Theorem \ref{thm:clip_rff} -- i.e., we do not call DPNorm in either of the algorithms.
 We compute the clustering accuracy for different values of $d$ depending on $\sigma_\mathcal{K}$, which directly yields the privacy bound $\varepsilon$ using Eq.~\ref{eq:ak} and Theorem~\ref{thm:comp}. Finally, we plot the average accuracy over 100 runs as function of $\varepsilon$ in Figure \ref{fig:cluster}.\footnote{Standard deviation of accuracy is $<$ 0.05 for all values of $\varepsilon$ and $d$.} Private kernel $k$-means is clearly superior to standard DP $k$-means, as the difference  in clustering accuracy can be as large as 20\%, especially for smaller values of $\varepsilon$.  
 Shorter RFF features (i.e., smaller $d$) result in larger accuracy for smaller values of $\varepsilon$, whereas the reverse holds for larger $\varepsilon$. The reason is that the clustering error is determined by the trade-off  between (1) the perturbation error due to the Gaussian noise, which is added to the cluster centers in Line 11 of Alg.~\ref{alg:dpkkmeans}, and (2) the approximation error caused by the low-dimensional embedding $z$ in Line 3 of Alg.~\ref{alg:dpkkmeans}.
 In particular, the perturbation error increases if $\varepsilon$ decreases or $d$ increases. Indeed, when the distance $||\hat{z}(\mathbf{x}) - \mathbf{\hat{c}}_j ||_2^2$ to each cluster center $ \mathbf{\hat{c}}_j$ is computed in Line 9  of Alg.~\ref{alg:dpkkmeans}, the total perturbation of this distance value is obtained by aggregating the noise values on each coordinate  of  $\mathbf{\hat{c}}_j$, and hence the perturbation error is proportional to the size $d$ of vector $\mathbf{\hat{c}}_j$ as well as to $\varepsilon^{-1}$. On the other hand, larger $d$ decreases the approximation error introduced by $z$. One can find a good trade-off between the approximation and the perturbation error by adjusting $d$ and $\varepsilon$ through experiments using publicly available data.
 For the rest of experiments, we set $d$ to 200. 
 
 Selecting the optimal number of clusters $k$ for kernel $k$-means can be qualitatively and visually done by relying on dimensionality reduction algorithms (e.g., t-SNE~\cite{maaten2008visualizing}). To this end, one can use public data sampled from the same underlying distribution, and therefore not requiring to make the parameter selection step differentially private. For MNIST we set $k=10$, while we select only one cluster for the CDR dataset. We investigate the effects of different values of $k$ for the transit dataset.

\descr{Synthetic Samples.}
As training progresses, the synthetic samples produced by the generative models should resemble the true samples. To evaluate model quality, we show the synthetic samples obtained at epoch 20  in \figurename~\ref{fig:samples} from a Restricted Boltzmann Machine and a Variational Autoencoder with $k=10$ clusters on MNIST. For this experiment, we set $q=0.0017$ for a final privacy budget $\varepsilon$ of $1.74$, and performed $T_{\mathcal{K}}=20$ clustering iterations before training the generative neural networks.
Overall, the samples generated from  VAE (Fig.~\ref{fig:vae_samples}) provide better visual quality than the ones generated from the RBM (Fig.~\ref{fig:rbm_samples}).
Note that the samples generated from the VAE without our private clustering technique (\figurename~\ref{fig:vae_no_cluster}) have bad visual quality.
Finally, we report additional samples with a multi-layer VAE in Appendix~\ref{app:add}.

\begin{figure*}[!t]
\centering
\begin{subfigure}[t]{0.9\columnwidth}
\centering
\includegraphics[width=0.775\columnwidth]{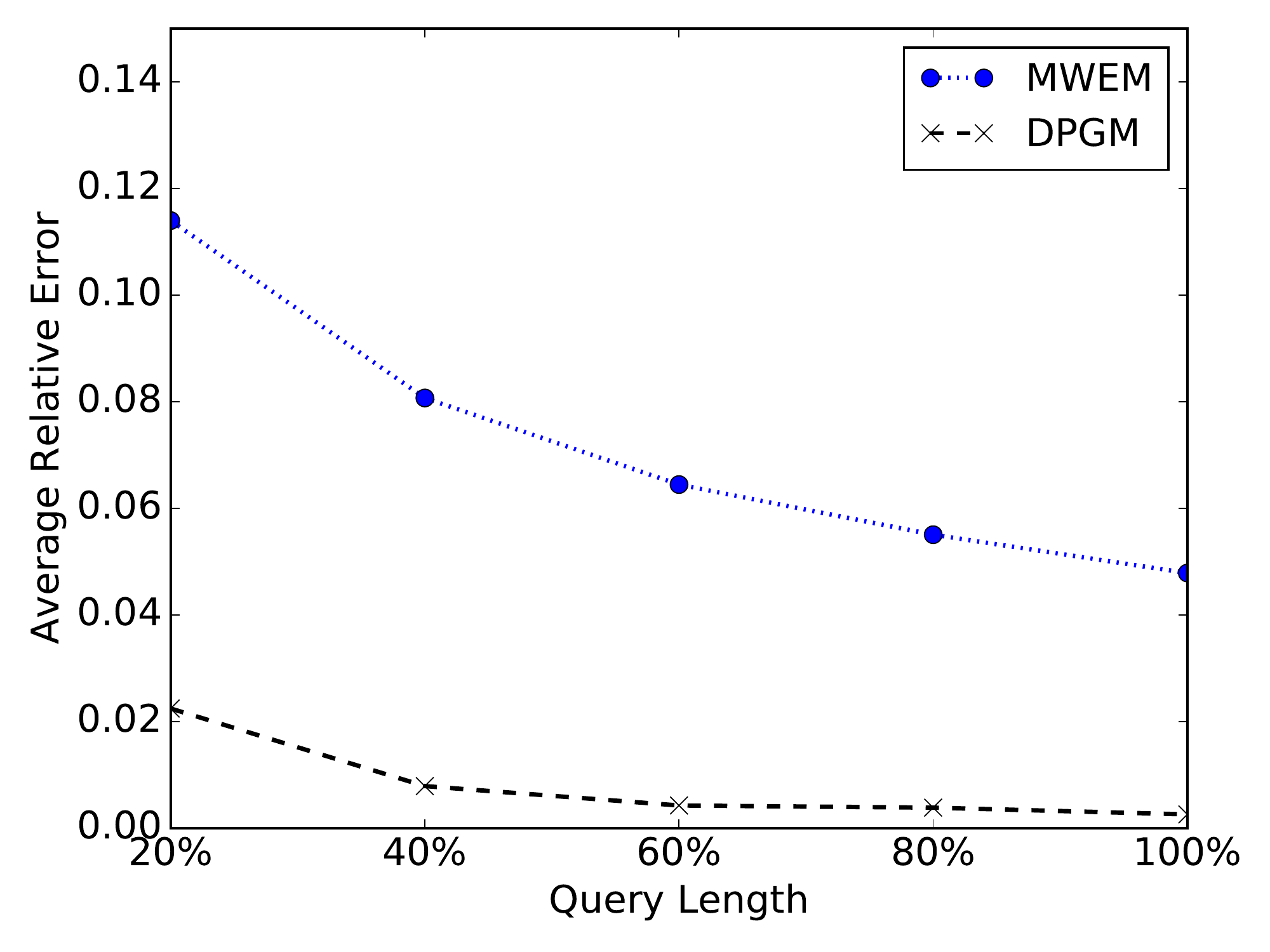}
\caption{$\varepsilon=2.0$}
\label{fig:queries_sigma1}
\end{subfigure}
\begin{subfigure}[t]{0.9\columnwidth}
\centering
\includegraphics[width=0.775\columnwidth]{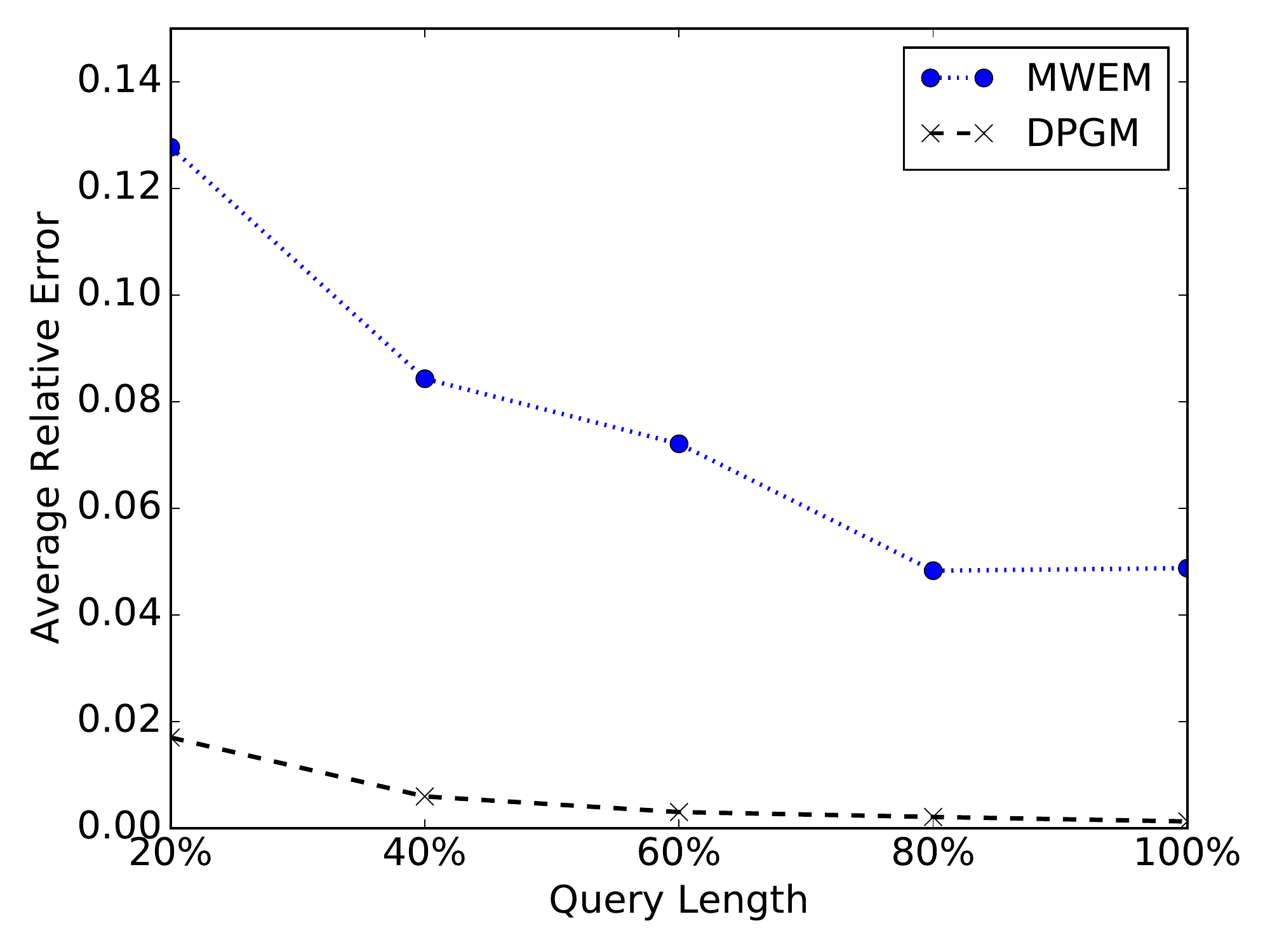}
\caption{$\varepsilon=1.0$}
\label{fig:queries_sigma2}
\end{subfigure}
\caption{Average relative error vs. $\varepsilon$ for the CDR dataset ($q = 2.2\cdot 10^{-5}, \delta = 4.4 \cdot 10^{-6}$)}
\label{fig:vae_queries}
\vspace{-0.1cm}
\end{figure*}

\begin{figure*}[!t]
\centering
\begin{subfigure}[t]{0.9\columnwidth}
\centering
\includegraphics[width=0.775\columnwidth]{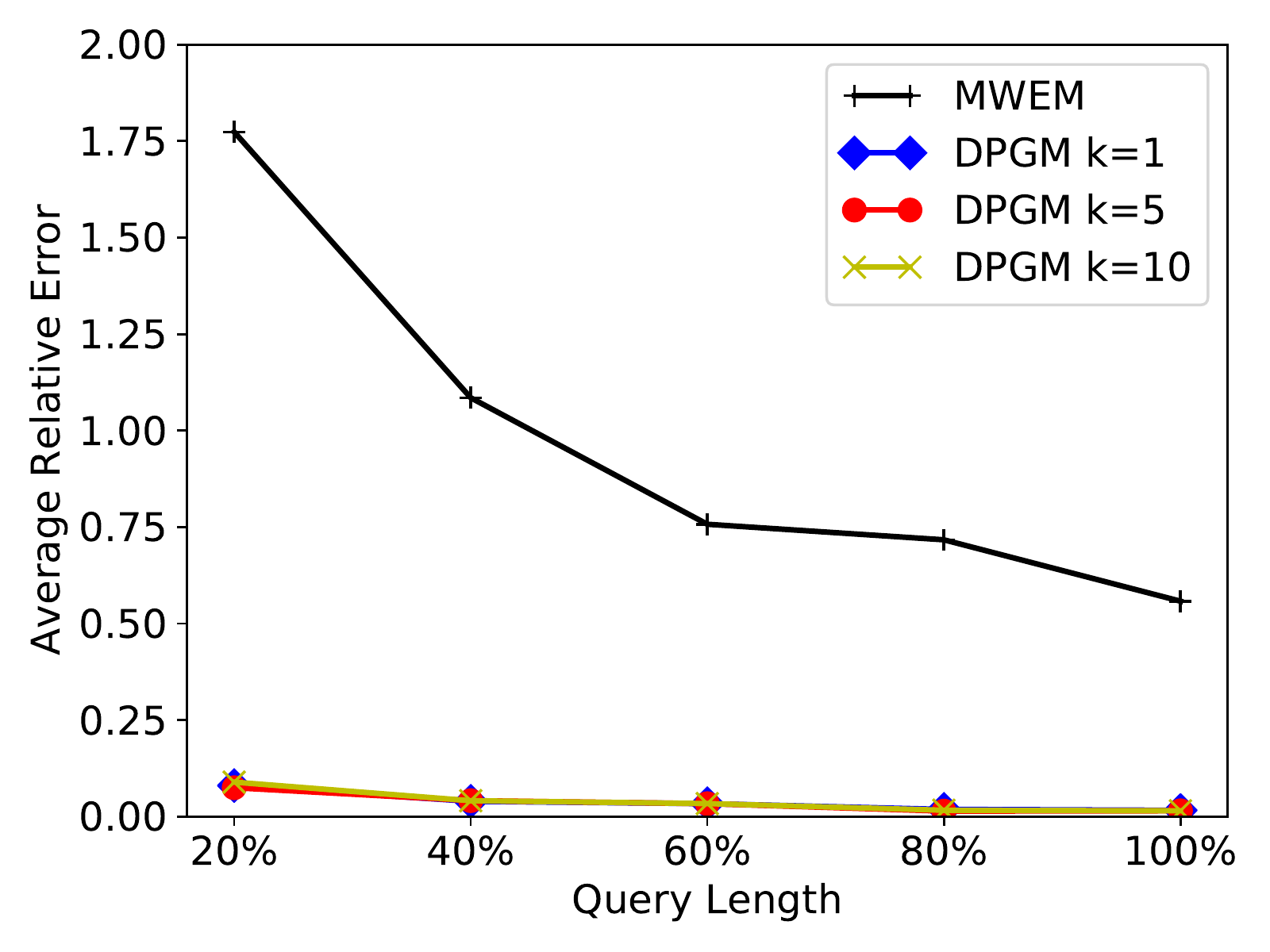}
\caption{$\varepsilon=2.0$}
\label{fig:stm_queries2}
\end{subfigure}
\begin{subfigure}[t]{0.9\columnwidth}
\centering
\includegraphics[width=0.775\columnwidth]{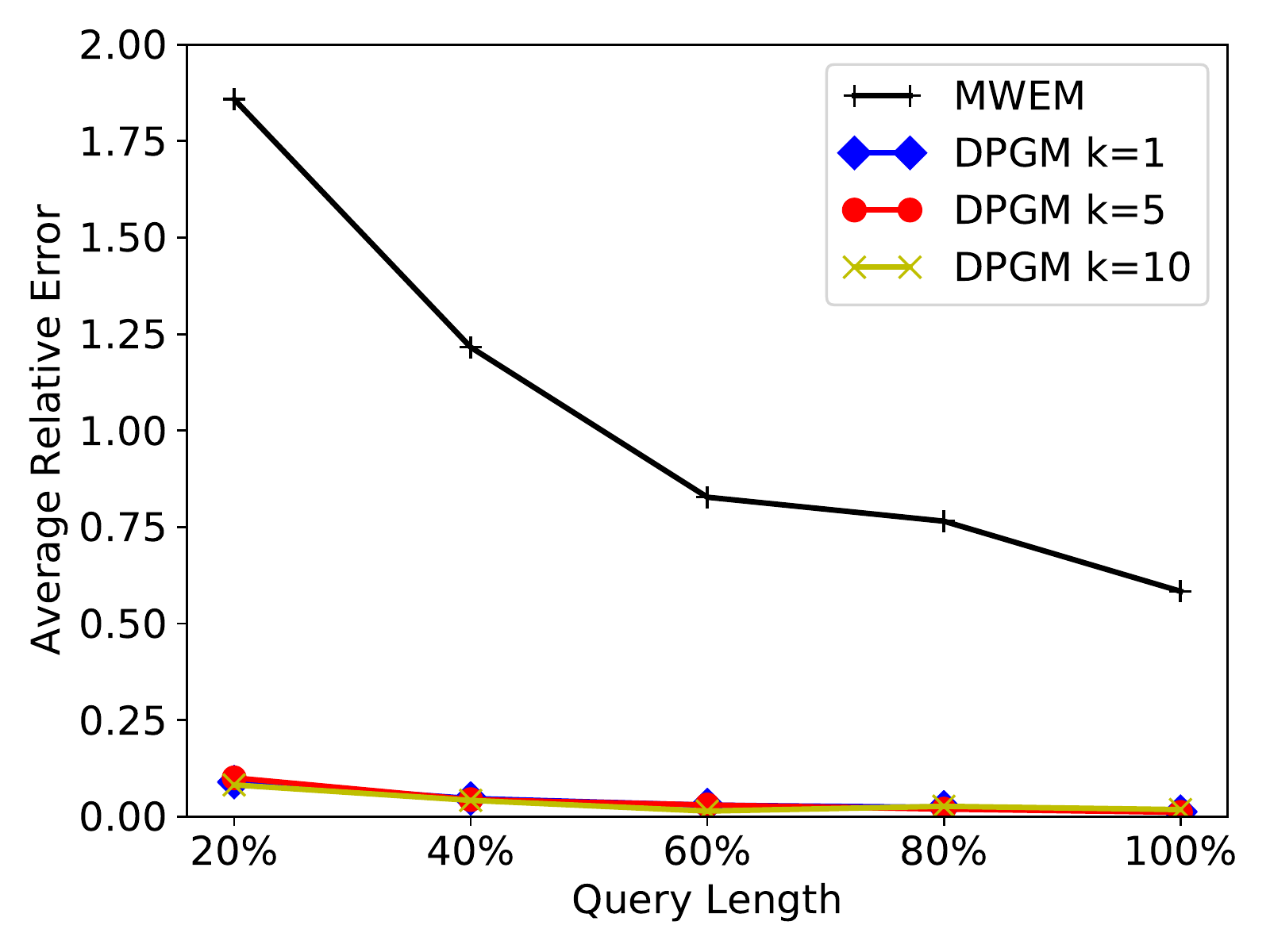}
\caption{$\varepsilon=1.0$}
\label{fig:stm_queries1}
\end{subfigure}
\caption{Average relative error vs. $\varepsilon$ for the transit dataset ($q = 10^{-4}, \delta = 10^{-6}$) }%
\label{fig:stm_queries}
\vspace{-0.1cm}
\end{figure*}

\subsection{Results with CDR and transit dataset}

We consider counting queries which are specified by a predicate function $p: D \rightarrow 
\{0,1\}$ and return the number of users in the dataset which satisfy the given predicate $p$, i.e., $Q_p(D) = \sum_{\mathbf{x} \in D} p(\mathbf{x})$. We evaluate the accuracy of  counting queries on a synthetic dataset generated by DPGM from our call-data-record (CDR) dataset with roughly 4 million users and the transit dataset with roughly 1 million users (see in Table~\ref{tab:datasets}). A single query is defined by a subset of tower cells, and returns the number of users in $D$ who visited these cells. We compare DPGM with MWEM \cite{HardtLM12}, which is a \emph{de facto} standard differentially private mechanism to answer counting queries.

As done in previous work~\cite{XiaoBHG11}, we measure the utility of a counting query $Q_p$ over the sanitized dataset $\hat{D}$ by its relative error w.r.t.~the actual result over the raw dataset $D$. The relative error of $Q_p$ is thus computed as $\frac{|Q_p(\hat{D}) - Q_p(D) |}{ \max\{Q_p(D),s\}}$, where $s$ is a sanity bound that weight the influence of the queries with small selectivities. %
Following the convention, the sanity bound is set to 0.1\% of the dataset size.

First, we examine the relative error of counting queries with respect to privacy loss $\varepsilon$. 
$1,000$ counting queries are randomly generated with different number of tower cells, which we refer as the length of the query. 
Each query set is divided into 5 subsets such that the query length of the $i$-th subset is uniformly distributed in $\left[1, \frac{i \cdot  \max ||\mathbf{x}||_1}{5} \right]$ and each item is randomly drawn from universe of items. 

Fig.~\ref{fig:vae_queries} reports the average relative error for each query set.
This shows that our approach clearly outperforms MWEM.
The error of DPGM ranges from $0.017$ for 20\% query length to $0.0012$ for 100\% when $\varepsilon=1.0$.
Weaker privacy guarantee (larger values of $\varepsilon$) lead to slightly smaller errors (Fig.~\ref{fig:queries_sigma2}). By contrast, the error of MWEM\footnote{After clipping each record to have $L_1$-norm $\mathrm{avg}||\mathbf{x}||_1 = 12$, the sensitivity of queries is set to 12, and the iterations of the algorithm is set to $50$ \cite{HardtLM12}.} ranges from 0.11 to 0.05 even for $\varepsilon=2$. Also note that the synthetic data produced by DPGM allows the evaluation of arbitrary number of type of queries, not only linear counting queries.

Finally, Fig.~\ref{fig:stm_queries} reports the average relative error for the transit dataset with different number of clusters $k$.
While our approach, whose average relative error ranges from $0.09$ to $0.02$, significantly outperforms MWEM, the number of clusters does not affect the error of counting queries on transit dataset.
These results might be an artifact of the dataset itself. 
We also report additional results with a multi-layer VAE for both the CDR and Transit dataset in Appendix~\ref{app:add}.

\section{Conclusion}
This paper presented a first-of-its-kind attempt to build private generative machine learning models based on neural networks.
Specifically, we presented a novel differentially private generative model (DPGM), relying on a mixture of $k$ 
generative neural networks: such models can be used to generate and share synthetic high-dimensional data with provable privacy.
We evaluated the performance of the model on real datasets, showing that our approach provides accurate representation of large datasets with strong privacy guarantees and high utility. 
As part of future work, we plan to combine VAE and a Gaussian Mixture model for clustering, similar to~\cite{zheng2016variational}, albeit with strong privacy guarantees. The effective privacy-preserving training of deep neural networks with multiple hidden layers is also desirable in order to generate more complex data such as personal photos or various sequential data. Finally, we plan to pilot deploy our techniques in the wild.

\descr{Acknowledgments.} Luca Melis and Emiliano De Cristofaro were partially supported by The Alan Turing Institute under the EPSRC grant EP/N510129/1 and a grant by Nokia Bell Labs, Gergely Acs by the Premium Post Doctorate Research Grant of the Hungarian Academy of Sciences (MTA) and the Higher Education Excellence Program of the Ministry of Human Capacities in the frame of  Artificial Intelligence research area of Budapest University of Technology and Economics (BME FIKP-MI/FM). Claude Castelluccia was supported by the French National Research Agency in the framework of the  ``Investissements d'avenir'' program (ANR-15-IDEX-02).

{\small
\bibliographystyle{abbrv}
\bibliography{bibfile}}

\appendix
\section{Multi-layer Variational Autoencoder}\label{app:add}

\begin{figure}[!t]
\centering
\begin{subfigure}[t]{0.22\textwidth}
\includegraphics[width=1\textwidth]{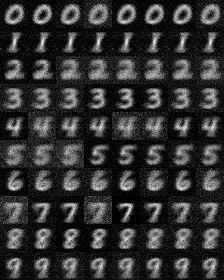}
\caption{$\varepsilon=1.74$}
\label{fig:2layers_samples1_0}
\end{subfigure}
~
\centering
\begin{subfigure}[t]{0.22\textwidth}
\includegraphics[width=1\textwidth]{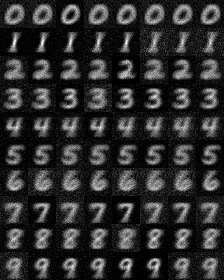}
\caption{$\varepsilon=2.0$}
\label{fig:2layers_samples1_2}
\end{subfigure}
\caption{Samples generated from a double layer VAE after 20 epochs. 
Each row contains 8 samples generated from a cluster.}
\label{fig:2layers_samples}
\end{figure}

\begin{figure}[!t]
\centering
\begin{subfigure}[t]{0.9\columnwidth}
\centering
\includegraphics[width=.725\columnwidth]{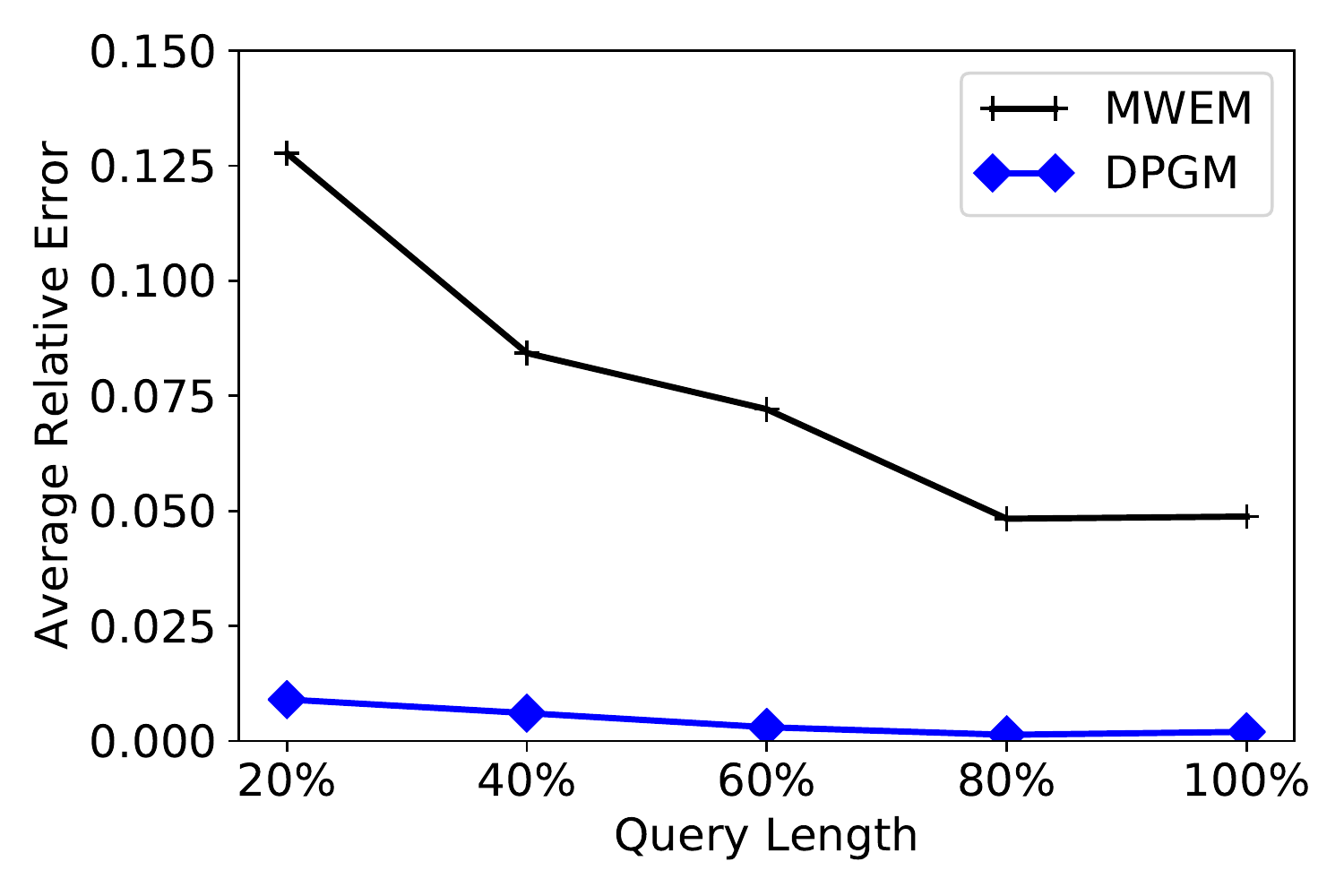}
\caption{CDR dataset}
\label{fig:2layers_orange}
\end{subfigure}
\begin{subfigure}[t]{0.9\columnwidth}
\centering
\includegraphics[width=.725\columnwidth]{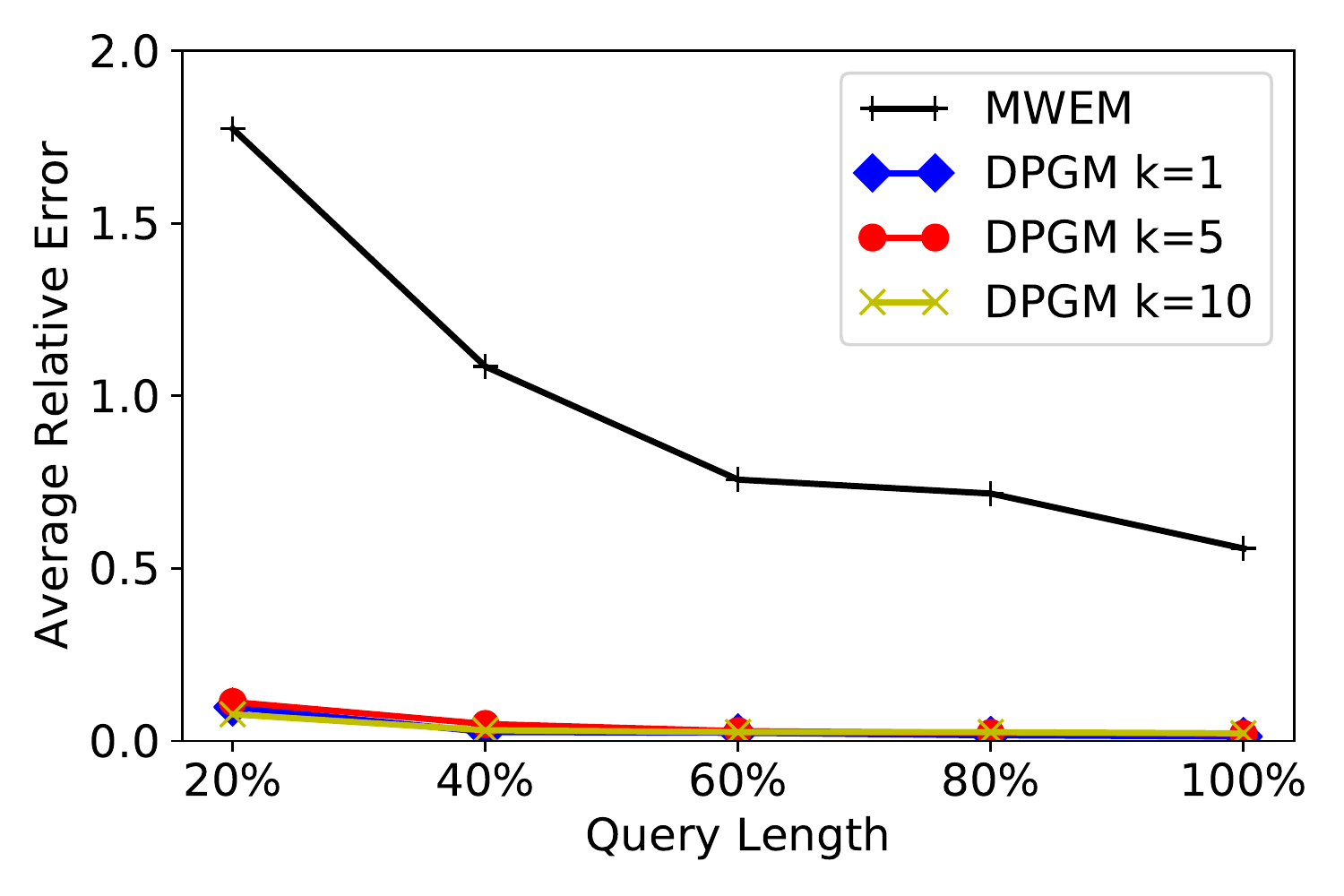}
\caption{Transit dataset}
\label{fig:2layers_stm}
\end{subfigure}
\vspace{-0.1cm}
\caption{Average relative error with $\varepsilon=1.0$ for the CDR and transit datasets.}%
\label{fig:2layers_queries}
\end{figure}

We now report additional results for a VAE with a double layer encoder and decoder.
In Fig.~\ref{fig:2layers_samples}, we show the synthetic samples obtained at epoch 20 from a VAE with $k=10$ clusters on MNIST. 

Then, Fig.~\ref{fig:2layers_orange} reports the average relative error for the CDR dataset, while Fig.~\ref{fig:2layers_stm} shows the average relative error for the transit dataset with different number of clusters $k$.

Overall, we can observe that increasing the number of layers, and thus the capacity of the VAE, does not lead to better performances.

\end{document}